\documentclass[12pt,reqno]{amsart}

\title[Asymptotic risk of overparameterized likelihood models]{Asymptotic risk of overparameterized likelihood models: \\ double descent theory for deep neural networks}

\usepackage{amsaddr}

\author{Ryumei Nakada$^1$ \and Masaaki Imaizumi$^{2,3}$}

\address{$^1$Rutgers University, $^2$The University of Tokyo, $^3$RIKEN AIP}
\thanks{\textit{Contact:} \texttt{imaizumi@g.ecc.u-tokyo.ac.jp}}

\date{\today}

\usepackage{fullpage}

\usepackage{amsmath, amssymb}
\usepackage{thmtools}
\usepackage{comment}
\usepackage{mathptmx}
\usepackage{physics}
\usepackage{mathtools}
\usepackage[utf8]{inputenc} 
\usepackage[T1]{fontenc}    
\usepackage{hyperref}       
\usepackage{url}            
\usepackage{booktabs}       
\usepackage{amsfonts}       
\usepackage{nicefrac}       
\usepackage{microtype}      
\usepackage{lipsum}
\usepackage{here}
\usepackage{fancyhdr}

\DeclareMathOperator{\Var}{Var}

\DeclareMathOperator{\Cov}{Cov}

\DeclareMathOperator{\esssup}{ess sup}

\DeclareMathOperator{\argmin}{argmin}

\DeclareMathOperator{\diag}{diag}

\DeclareMathOperator{\card}{card}

\DeclareMathOperator{\sign}{sign}

\newcommand{\Ep}{\mathbb{E}}
\newcommand{\Pp}{\mathbb{P}}

\newcommand{\1}{\mbox{1}\hspace{-0.25em}\mbox{l}}

\newcommand{\R}{\mathbb{R}}

\newcommand{\N}{\mathbb{N}}

\newcommand{\mG}{\mathcal{G}}

\newcommand{\mN}{\mathcal{N}}

\newcommand{\mR}{\mathcal{R}}

\newcommand{\mZ}{\mathcal{Z}}

\newcommand{\indep}{\mathop{\perp\!\!\!\!\perp}}

\newcommand{\mone}{\mathbf{1}}

\renewcommand{\tilde}{\widetilde}
\renewcommand{\hat}{\widehat}

\usepackage[numbers]{natbib}

\usepackage{subcaption}
\usepackage{caption}

\usepackage[utf8]{inputenc} 
\usepackage[T1]{fontenc}    
\usepackage{hyperref}       
\usepackage{url}            
\usepackage{booktabs}       
\usepackage{amsfonts}       
\usepackage{nicefrac}       
\usepackage{microtype}      
\usepackage{lipsum}
\usepackage{graphicx}
\usepackage{physics}
\usepackage{color}
\newcommand{\rc}{\color{red}}
\newcommand{\bc}{\color{blue}}

\newtheorem{theorem}{Theorem}

\newtheorem{corollary}{Corollary}
\newtheorem{lemma}{Lemma}
\newtheorem{proposition}{Proposition}
\newtheorem{assumption}{Assumption}
\newtheorem{example}{Example}

\theoremstyle{remark}
\newtheorem*{remark}{Remark}

\allowdisplaybreaks

\begin{document}

\maketitle

\begin{abstract}
We investigate the asymptotic risk of a general class of overparameterized likelihood models, including deep models. The recent empirical success of large-scale models has motivated several theoretical studies to investigate a scenario wherein both the number of samples, $n$, and parameters, $p$, diverge to infinity and derive an asymptotic risk at the limit. However, these theorems are only valid for linear-in-feature models, such as generalized linear regression, kernel regression, and shallow neural networks. Hence, it is difficult to investigate a wider class of nonlinear models, including deep neural networks with three or more layers. In this study, we consider a likelihood maximization problem without the model constraints and analyze the upper bound of an asymptotic risk of an estimator with penalization. Technically, we combine a property of the Fisher information matrix with an extended Marchenko–Pastur law and associate the combination with empirical process techniques. The derived bound is general, as it describes both the double descent and the regularized risk curves, depending on the penalization. Our results are valid without the linear-in-feature constraints on models and allow us to derive the general spectral distributions of a Fisher information matrix from the likelihood. We demonstrate that several explicit models, such as parallel deep neural networks, ensemble learning, and residual networks, are in agreement with our theory. This result indicates that even large and deep models have a small asymptotic risk if they exhibit a specific structure, such as divisibility. To verify this finding, we conduct a real-data experiment with parallel deep neural networks. Our results expand the applicability of the asymptotic risk analysis, and may also contribute to the understanding and application of deep learning.
\end{abstract}

\section{Introduction}

We investigate a likelihood optimization problem in the overparameterized regime.
Using a $p$-dimensional parameter $\theta \in \R^p$, we consider a probability density function $f_\theta(z)$ of a sample $z$, characterized by $\theta$.
We observe $n$ i.i.d. observations, $z_1,...,z_n$, generated from $f_{\theta^*}$ with a true parameter $\theta^*$.
With respect to $f_\theta(z)$ as a likelihood function of $\theta$, given $z$, we consider a maximum likelihood estimator with penalization, which can be defined as
\begin{align}
    \hat\theta = \argmin_\theta -\frac{1}{n} \sum_{i=1}^n \log f_\theta(z_i) + \frac{\tau}{2} \norm{\theta}^2, \label{def:model1}
\end{align}
where $\tau > 0$ is a regularization coefficient and $\|\cdot\|$ is the $\ell 2$-norm.
Our goal is to analyze the discrepancy between $\Hat{\theta} $ and $\theta^*$ in the overparameterized asymptotics; hence, we consider a limit $p,n \to \infty$, while $p/n \to \gamma$ holds with a ratio $\gamma \in (0,\infty)$.
We show that the estimation risk is bounded by an extended Stieltjes transform of the spectral measure of an asymptotic Fisher information matrix.
The derived bound describes both a double descent and a regularized asymptotic risk, depending on the limit of $\tau$.
We achieve these evaluations without imposing constraints on the model.

Inspired by the success of deep learning \citep{lecun2015deep}, there is a growing interest in investigating the properties of large-scale statistical models; however, this presents a significant challenge to the statistics and learning theory (for keen discussion, see \citep{zhang2016understanding,frankle2018lottery,nagarajan2019uniform}). 
The classical learning theory states that models with a large number of parameters may perform poorly owing to overfitting with training data. 
However, actual deep learning achieves good generalization performance, despite having a large number of parameters. 
To resolve this discrepancy between theory and practice, numerous studies have strived to rethink the generalization of deep learning. 
An example is the complexity assessment of norm-constrained neural networks and the implicit regularization theory, which considers the influence of learning algorithms \citep{neyshabur2015norm,bartlett2017spectrally}.

\textit{Asymptotic risk analysis with overparameterization} has received considerable attention as a theory for describing statistical models with an excessive number of parameters.
One typical result is the \textit{double descent} phenomenon.
It demonstrates that a risk begins to decrease when the number of parameters exceeds a certain threshold.
From an experimental point of view, \cite{vallet1989linear,belkin2019reconciling} demonstrated that the phenomenon of double descent is observed in simple models, such as neural networks with two layers, and \cite{nakkiran2019deep} found that similar phenomena occur in deep neural networks with more than three layers.
With regard to the theory on the double descent, linear regression \cite{advani2020high,hastie1987generalized}, kernel/feature regression \cite{belkin2019reconciling,mei2019generalization,liang2020just,d2020double,hu2020universality,jacot2020implicit,jacot2020kernel}, classification with generalized linear models \cite{montanari2019generalization,deng2019model,kini2020analytic} have been studied.
For shallow neural networks, which likewise represent a linear-in-feature model, connections were derived by \cite{hastie1987generalized,ba2019generalization}. 
Detailed characteristics, such as the effect of the activation functions and number of descents, have been studied \cite{dar2020subspace,spigler2019jamming,liang2020multiple,d2020triple}.
A brief history of this phenomenon has been presented by \cite{loog2020brief}.
As another direction, the regularized asymptotic risk has been actively studied as well.
A typical example is a risk of ridge regression in the overparameterized limit.
\cite{donoho2016high,dobriban2018high,lolas2020regularization} derived an asymptotic risk of linear regression and classification with $\ell 2$ regularization with the random matrix theory, and \cite{bartlett2020benign, tsigler2020benign} investigated a linear regression model using the notion of effective ranks.
\cite{wu2020optimal,kobak2018optimal} analyzed the optimal regularization for linear regression with a ridge penalty.
Both theories affirm that large-scale models can achieve small risks, even when the number of parameters is infinitely large.

One critical challenge of asymptotic risk analysis is the study of general nonlinear models, including deep neural networks. 
In previous studies, theories are capable of analyzing only relatively simple nonlinear models, specifically, \textit{linear-in-feature} models.
These studies consider the following optimization problem with the models:
\begin{align}
    \min_{w = (w_1,...,w_p)} \sum_{i=1}^n \ell \left( Y_i , g_w(X_i)\right) + \Omega(w), \mbox{~s.t.~}g_w(X_i) = \sum_{j=1}^p w_j \Psi_j(X_i), \label{eq:exist_optimization}
\end{align}
where $(X_i,Y_i),i=1,...,n$ are observations, $w_j$ is a parameter, $\Psi_j$ is a given feature function, $\ell(\cdot,\cdot)$ is a loss function, $\Omega(\cdot)$ is a (possibly zero) regularization term, and $g_w(\cdot)$ is a linear-in-feature model.
For example, with $\ell (u,u') = (u-u')^2$, the problem represents a linear regression ($\Psi_j(x) = x_{j}$ is a covariate vector), a kernel regression ($\Psi_j(x)$ is a kernel function $k(x, X_j)$), and a two-layer neural network ($\Psi_j(x)$ is a neural map in a first layer with random weights).
By changing the setting of $\ell(\cdot,\cdot)$, this form can also represent a generalized linear regression and a classification problem.
Owing to linear-in-feature form constraints $g_w(X_i) = \sum_{j=1}^p w_j \Psi_j(X_i)$, the previous asymptotic risk analysis theory cannot deal with deep models that contain more than two layers of trainable parameters. 
This limitation is a result of the mathematical tool used for their proof, such as the random matrix theory, which heavily depends on the linear-in-feature structure.
Therefore, it is not clear whether existing asymptotic risk analysis can explain the performance of deep neural networks and other complicated models.

\subsection{Result Overview}
We develop an asymptotic risk bound for a wide range of models with likelihood maximization. Let $F_p^*$ be a Fisher information matrix, which is positive definite, and $\|\theta\|_{F_p^*}^2 := \theta^\top {F_p^*} \theta$ be a weighted $\ell 2$ norm.
We study the estimation risk, $\|\Hat{\theta} - \theta^*\|^2_{F_p^*}$, in the setting $n,p \to \infty$ and $p/n \to \gamma$.
We set $\|\theta^*\|^2 \leq r^2$ for $r > 0$.
For the analysis, we divide the risk, $\hat{\theta} - \theta^* = V_{n,p}(\hat{\theta}) + B_{n,p}(\Hat{\theta})$, into variance $V_{n,p}(\hat{\theta})$ and bias $B_{n,p}(\Hat{\theta})$, details of which will be provided in Section \ref{sec:main_statement}.
$\lesssim $ denotes an inequality up to constant factors.

The contributions of this study consist of two aspects.
First, we analyze the asymptotic estimation risk of the penalized maximum likelihood estimator in \eqref{def:model1} at the limit. 
To achieve this, we require mainly two assumptions in addition to the basic regularity conditions: (i) a weighted derivative of the likelihood function is bounded and has a pointwise second-order moment (Assumption \ref{asmp:fisher_residual}), and (ii) a sum of the third-order derivatives of the likelihood function diverges slower than $p$ (Assumption \ref{asmp:taylor_residual}).
Under these assumptions, we informally obtain the following results:
\begin{theorem}[informal statement of Theorem \ref{thm:variance} and \ref{thm:main}]
    Suppose the assumptions hold.
    Then, with the penalized parameter $\tau \searrow \overline{\tau} \geq 0$, we obtain
    \begin{align*}
        \limsup_{p,n \to \infty, p/n \to \gamma \in (0,\infty)}\|V_{n,p}(\Hat{\theta})\|_{F_p^*}^2 \lesssim \lim_{a \to 0} h_{\gamma, \overline{\tau}}(a),
    \end{align*}
    where $h_{\gamma, \overline{\tau}}$ is the extended Stieltjes transformation of a spectral measure of $\lim_{p \to \infty} (F_p^*)^{-1}$, which will be defined in Section \ref{sec:main_statement}.
    Furthermore, for the estimation risk, we obtain
    \begin{align*}
        \limsup_{p,n \to \infty, p/n \to \gamma \in (0,\infty)}\|\Hat{\theta} - \theta^*\|_{F_p^*}^2 \lesssim \lim_{a \to 0} h_{\gamma, \overline{\tau}}(a) + r^2.
    \end{align*}
\end{theorem}
The extended Stieltjes transform, $\lim_{a \to 0} h_{\gamma, \overline{\tau}}(a)$, is a general term that describes both the double descent and the regularized asymptotic risk curve, depending on $\overline{\tau}$.
When $\overline{\tau} = 0$, $\lim_{a \to 0} h_{\gamma, \overline{\tau}}(a)$ presents the double descent curve in $\gamma$ in Figure \ref{fig:risk_intro}, which exactly corresponds with the curve derived by \cite{hastie2019surprises} for the linear regression case.
When $\overline{\tau} > 0$, the value of $\lim_{a \to 0} h_{\gamma, \overline{\tau}}(a)$ has a different curve in Figure \ref{fig:risk_intro2}, which is analogous to the asymptotic risk curve of high-dimensional ridge regression, which was analyzed by \cite{dobriban2018high,hastie2019surprises}.

This result has several merits. 
(i) It holds for general nonlinear models, which are not necessarily linear-in-features, as long as our assumptions are satisfied. 
This is also valid for deep models, which were not considered in previous studies. 
(ii) It describes the two different risk curves in a unified manner. 
In either case, the result is consistent with the previous studies on linear regression.
(iii) The result is also valid with arbitrary spectral measures of the positive definite Fisher information matrix.
Because the previous studies only deal with a Dirac measure, our results can be regarded as a generalization of them.

\begin{figure}[htbp]
    \centering
    \begin{minipage}{0.47\hsize}
    \centering
    \captionsetup{width=0.9\hsize}
    \includegraphics[width=0.75\hsize]{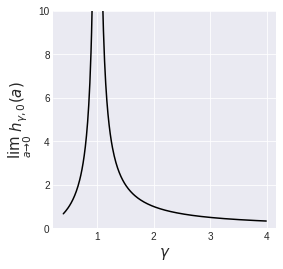}
    \caption{$\lim_{a \to 0} h_{\gamma, \overline{\tau}}(a)$ under $\overline{\tau} = 0$ with respect to ratio $\gamma$ such that $p/n \to \gamma$. \label{fig:risk_intro}}
    \end{minipage}
    \begin{minipage}{0.47\hsize}
    \centering
    \captionsetup{width=0.9\hsize}
    \includegraphics[width=0.75\hsize]{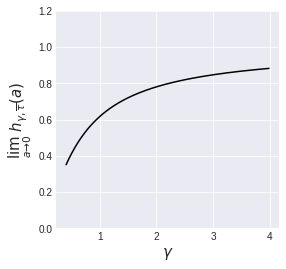}
    \caption{$\lim_{a \to 0} h_{\gamma, \overline{\tau}}(a)$ under $\overline{\tau} > 0$ with respect to ratio $\gamma$. We select $\overline{\tau } = \gamma$. \label{fig:risk_intro2}}
    \end{minipage}
\end{figure}

With regard to the second aspect of our contribution, we specify a statistical model and a learning scheme that comply with our asymptotic risk analysis.
To achieve this, we verify that models such as the (i) standard linear regression model, (ii) deep neural network with parallel structure, (iii) ensemble learning scheme, and (iv) residual networks satisfy our assumptions on higher-order derivatives.
To prove the validity of the (ii) parallel deep neural network and (iii) ensemble learning, we show that the assumptions are satisfied by the following regression model with the $M$ additive structure:
\begin{align*}
    y=\frac{1}{M}\sum_{m=1}^M g_{m}(x) + \varepsilon,
\end{align*}
where $x$ is the covariate, $y$ is the response, $\varepsilon$ is the noise, and $g_1,..., g_M$ denote the submodels.
This additive nonparametric model describes the parallel deep neural network in Figure \ref{fig:para_network}.
Further, this result implies that (iv) residual networks (ResNet) also follow our theoretical result.
These results demonstrate certain deep model compliance with the asymptotic risk analysis for the first time.

\begin{figure}[htbp]
    \centering
    \includegraphics[width=0.5\hsize]{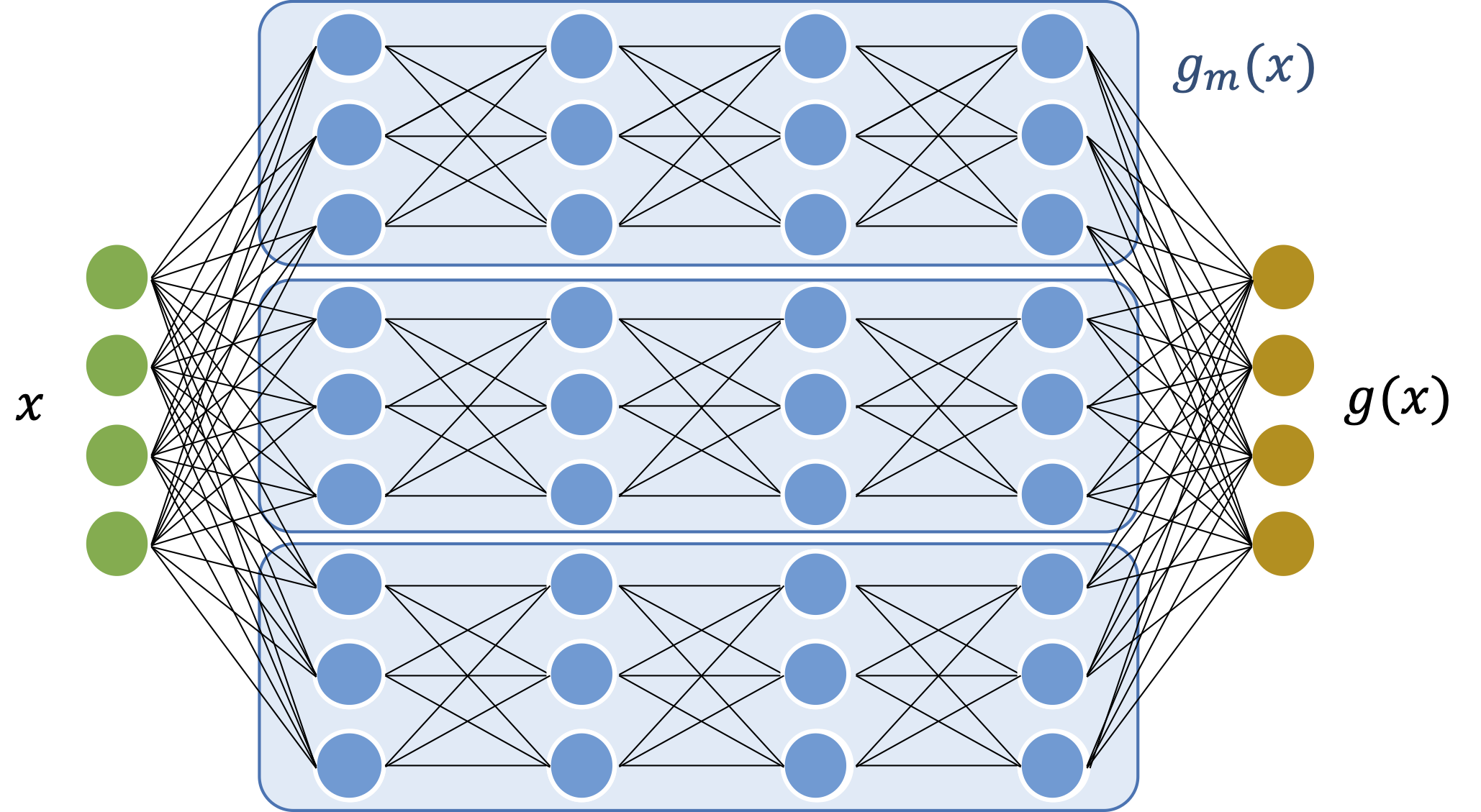}
    \caption{Parallel deep neural network with five layers. Nodes are divided into $M$ groups with $M=3$. This network represents a function $g(x) = M^{-1} \sum_{m=1}^M g_m(x)$.}
    \label{fig:para_network}
\end{figure}

\subsection{Technical Point}
The key technique that we employ to achieve our result is the combination of a decomposition of the empirical Fisher information matrix and the Marchenko–Pasteur law \cite{marvcenko1967distribution}.
We start with simple linear regression with a squared loss case in the previous study.
\cite{hastie2019surprises} derived that a variance term of an risk is given by the trace of a product of random matrices.
Rigorously, let $x_i, i = 1,...,n$ be a $p$-dimensional random variable with zero mean and an identity covariance matrix, and $X = (x_1,...,x_n)^\top$ be a $p \times n$ random design matrix.
In this case, the variance term is written as
\begin{align*}
    \frac{\sigma^2}{n}\mathrm{tr} \left(  \left( \frac{X^\top X}{n}\right)^{-1}  \right) = \frac{\sigma^2 p}{n} \int \frac{1}{t} dF_{X ^\top X/n}(t),
\end{align*}
where $\sigma^2$ is the noise variance, and $F_{X^\top X/p}$ is a spectral measure of random matrix $X^\top X/n$.
The product, $X^\top X$, is obtained from a second derivative of the loss function.
The integral in the right-hand side converges to the Stieltjes transform of the Marchenko–Pastur law as $p,n\to \infty$, whose form has been investigated extensively (for an overview, see \cite{bai2010spectral}).
Thus, we can obtain a tight evaluation of its asymptotic risk.

Our approach achieves a similar variance form by utilizing the basic principle of the maximum likelihood method on Fisher information matrices $F^*_p := \Ep[(\partial_\theta \log f_{\theta^*}(z))(\partial_\theta \log f_{\theta^*}(z))^\top] = -\Ep[\partial_\theta \partial_\theta^\top \log f_{\theta^*}(z)]$.
Hence, we can approximate $\partial_\theta \partial_\theta^\top (-n^{-1} \sum_{i=1}^n \log f_\theta(z_i))$ by $\hat{J}_{n,p} \hat J_{n,p}^\top / n^2$, where $\hat J_{n,p}$ is the $p \times n$ random matrix, whose $(j,i)$-element is $\partial_{\theta_j} f_{\theta^*}(z_i)$.
Using the approximation, we can achieve the following variance form:
\begin{align*}
  \|V_{n,p}(\Hat{\theta})\|_{F_p^*}^2 \lesssim \frac{1}{n}\tr\qty({F^*_p} \qty(\frac{\hat J_{n,p} \hat J_{n,p}^\top}{n} + \tau I_p)^{-1} ) + T_R =  \int \frac{1}{\lambda} \dd{\mu_{n,p}(\lambda)} + T_R,
\end{align*}
where $T_R$ is a residual term from approximations on models, and $\mu_{n,p}$ is a spectral measure of the regularized random matrix.
Owing to the form that is suitable for the random matrix theory, we obtain a limit of the right-hand side using the extended Marchenko–Pastur law by \cite{pajor2009limiting}.

Another important technical point is rendering the term $T_R$ asymptotically negligible at the overparameterized limit. 
This term is a measure of the degree of nonlinearities and complexities of models.
To make it vanish asymptotically, we represent $T_R$ by a decomposition error of the Fisher information matrix and residuals of the Taylor expansion of the likelihood function, then evaluate them using empirical processes and matrix concentration inequalities. 
As a result, we show that the second and third derivatives of a likelihood function represent the term, then derive sufficient conditions to make $T_R$ vanish in the limit.

\subsection{Notation}

We equip a usual $\ell2$ norm $\norm{b}^2 = b^\top b$ and a max norm $\|b\|_\infty = \max_{j = 1,...,p} |b_j|$ for a vector $b = (b_1,...,b_p) \in \R^p$.
With a weight matrix $A \in \R^{p \times p}$, $\norm{b}_A^2 = b^\top A b$ denotes a weighted $\ell 2$ norm.
For a matrix $A$, $\norm{A}_{op} = \sup_{x: \norm{x}=1} \norm{Ax}$ is denoted as the operator norm with the underlying $\ell2$ norm.
When $A$ is a symmetric matrix, $\lambda_i(A)$ denotes an $i$-th largest eigenvalue of $A$.
Further, $\lambda_{\max}(A)$ and $\lambda_{\min}(A)$ denote the maximum and minimum eigenvalues, respectively. 
$\mathrm{rank}(A)$ denotes a rank of $A$, and $\lambda_{\min}^{+}(A)$ denotes a minimum positive eigenvalue of $A$.
For a function $f(x)$, $\partial_x f(\bar{x})$ denotes a partial derivative of $f$ in terms of $x$ at $\Bar{x}$.
$\|f\|_{L^\infty} = \sup_{x} |f(x)|$ is a sup-norm.
$I_n$ denotes an $n \times n$ identity matrix.
For $x \in \R^p$ and radii $r > 0$, $B_p(x,r) := \{x' \in \R^p \mid \|x-x'\|\leq r\}$ denotes a ball in $\R^p$ around $x$ with radius $r$.
We consider $a_{n,p} \succ b_{n,p}$ when $a_{n,p} / b_{n,p} \to \infty$.
We also consider $a_{n,p} \gtrsim b_{n,p}$ when there is some constant $c > 0$, such that $a_{n,p} / b_{n,p} \geq c$ for all $n, p$.
For an event $\omega$, $\mone(\omega)$ is an indicator function, which is $1$ if $\omega$ is true, and $0$ otherwise.

\section{Problem Setting and Basic Decomposition}

\subsection{Problem Setting: Penalized Likelihood Optimization}

We consider a likelihood maximization problem with $\ell 2$ penalization while increasing the dimension of parameters.
Let $\Theta \subset \R^p$ be a parameter space with its dimension $p$.
We consider a family of sufficiently smooth probability densities $\qty{f_\theta \mid \theta \in \Theta \subset \R^p}$, and assume samples $z_1, \dots, z_n \in \mZ \subset  \R^d$ are independently drawn from a probability density $f_{\theta^*}$ with a true parameter $\theta^* \in \Theta$.
Then, we rewrite the definition of the estimator \eqref{def:model1} into a simpler form
\begin{align}
    \Hat{\theta} = \argmin_{\theta \in \Theta} M_n(\theta) + \frac{\tau}{2}\|\theta\|^2, \label{def:estimator}
\end{align}
where $M_n(\theta) = -n^{-1} \sum_{i=1}^n \log f_\theta(z_i)$ is the mean of negative log-likelihood, and $\tau > 0$ is a coefficient for penalization.
Note that the density function $f_\theta$ can be any function of a parameter $\theta$.

Our goal is to clarify an asymptotic estimation and generalization risk of $\hat{\theta}$ at the limit $n,p \to \infty$ while preserving $p/n \to \gamma \in (0,\infty)$.
We study the estimation risk $\|\hat{\theta} - \theta^*\|_{F^*_p}^2$ with a (population) Fisher information matrix ${F^*_p} =  \mathbb{E}[ (\partial_\theta \log f_{\theta^*}(z)) (\partial_\theta \log f_{\theta^*}(z))^\top]$.
We also define an empirical Fisher information matrix as $\hat{F}_{n,p} = n^{-1} \sum_{i=1}^n (\partial_{\theta} \log f_{\theta^*}(z_i))(\partial_{\theta} \log f_{\theta^*}(z_i))^\top$.

\begin{remark}[Multiple true parameters]
We allow the model to have multiple true parameters $\theta^*$.
In such a case, we can pick any one of them. 
It should be noted that each $\theta^*$ is an isolated point by a condition on the Fisher information matrix, which will be presented in Section \ref{sec:assumption}.
\end{remark}

\subsection{Bias-Variance Decomposition}
We derive a simple representation of the discrepancy $\theta^* - \hat\theta$ by the bias-variance decomposition.
Because $\hat{\theta}$ is the minimizer in \eqref{def:estimator}, the Taylor's theorem provides the second order expansion as
\begin{align*}
    0 &= \partial_\theta \qty(M_n(\hat{\theta}) + \frac{\tau}{2}\|{\hat{\theta}}\|^2) \\
    &=\partial_\theta M_n(\theta^*)+ \partial_\theta \partial_\theta^\top M_n(\theta^*) (\hat\theta - \theta^*) + R + \tau(\hat\theta - \theta^*) + \tau\theta^*,
\end{align*}
where $R = (R_1, \dots, R_p)^\top$ is a \textit{Taylor residual} term with $R_j = (\hat\theta-\theta^*)^\top \partial_\theta\partial_\theta^\top \partial_{\theta_j} M_n(\check{\theta}^j) (\hat\theta - \theta^*)$ with an existing parameter $\check\theta^j = (\check\theta_1,...,\check\theta_p)^\top$ such that $\check\theta_k^j $ lies in the interval between $\hat\theta_k$ and $\theta^*_k$.
A simple calculation provides the following decomposition:
\begin{align}
     \theta^* - \hat\theta  = \underbrace{ (\partial_\theta \partial_\theta^\top M_n(\theta^*) + \tau I_p)^{-1} (\partial_\theta M_n(\theta^*) + {R})}_{=: V_{n,p}(\hat{\theta})} +  \underbrace{(\partial_\theta \partial_\theta^\top M_n(\theta^*) + \tau I_p)^{-1} \tau \theta^*}_{=: B_{n,p}(\hat{\theta})}, \label{def:bias_var}
\end{align}
where $V_{n,p}(\hat{\theta})$ denotes its variance and $B_{n,p}(\hat{\theta})$ is of bias.


\section{Assumption and Main Result}
We present rigorous assumptions and our main formal statement.
A relation between them will be described in the proof outline in Section \ref{sec:proof_outline}.

\subsection{Assumption} \label{sec:assumption}


\subsubsection{Basic Assumption}
First, we give fundamental assumptions regarding the estimation problem.
The following assumptions have generally been used in a wide range of studies.
\begin{assumption}[Basic] \label{asmp:basic}
    The following conditions hold:
    \begin{enumerate}
  \setlength{\parskip}{0cm}
  \setlength{\itemsep}{0cm}
        \item[(i)] $f_\theta$ is thrice differentiable with respect to all $\theta \in B(\theta^*, \varepsilon)$ with some $\varepsilon > 0$. \label{asmp:dif}
        \item[(ii)] A random variable $Y^* = \partial_\theta \log f_{\theta^*}(Z)$ has a probability density function that is log-concave; the density has a form $\exp(-\ell(\cdot))$ for some convex function $\ell$. \label{asmp:lconcave}
        \item[(iii)] $\theta^* \in B_p(0, r)$ and $\hat{\theta} \in B_p(0, r)$ hold, with probability approaching $1$ as $n,p \to \infty$ with $r > 0$.
        \item[(iv)] There exist constants $\underline{\lambda}$ and $\overline{\lambda}$ such as $0 < \underline{\lambda} \leq \lambda_{\min}(F_p^*) \leq \lambda_{\max}(F_p^*) \leq \overline{\lambda} < \infty$. \label{asmp:regular_hesse}
        \item[(v)] $\rank(\hat F_{n,p}) = \min\qty{n, p}$ holds and there exists a constant $\lambda_+ > 0$ such as $\lambda_{\min}^{+}(\hat F_{n,p}) \geq \lambda_+ $ for sufficiently large $n$ and $p$. 
    \end{enumerate}
\end{assumption}
Condition (i) requires the local differentiability of $f_\theta$.
This differentiability is necessary only in the neighborhood of $\theta^*$.
For example, even if the model is a deep neural network with non-smooth activation function, this condition is satisfied in many cases.
Condition (ii) is common, and a wide range of distributions satisfy it, including the exponential family.
Condition (iii) requires compactness of the parameter space, which is common in asymptotic risk analysis, e.g., in \cite{hastie2019surprises}.
Condition (iv) represents the positive definiteness of $F^*$.
This condition is necessary for numerous asymptotic risk analyses, such as \cite{dobriban2018high,hastie2019surprises,bartlett2020benign}.
Because the condition concerns the population Fisher information matrix, it is satisfied easily, unlike the empirical version of the Fisher matrix.
Condition (v) is a technical requirement and it holds associated with the property of its limit by Condition (iv).

\begin{remark}[Example of Positive Definite $F_p^*$]
In a linear regression model with Gaussian noise, the condition (iv) is satisfied if the covariance matrix of covariates is non-degenerate, since $F_p^*$ corresponds to it. 
Rigorously, a Fisher information matrix with a linear model
    $y = x^\top \theta^* + \epsilon$
with independent noise $\epsilon \sim \mN(0, \sigma^2)$ is written as $F_p^* = E[x x^\top] / \sigma^2$.
More generally, for a regression model $y = g_\theta(x) + \epsilon$, the condition (iv) is satisfied if $E[\partial_{\theta} g_{\theta^*} (x) \partial_{\theta}^\top g_{\theta^*} (x)]$ is non-degenerate.
In the case of deep neural networks, it is necessary to utilize a non-first-order homogeneous activation function, such as a sigmoid or soft-ReLU function and some requirements on the true parameter $\theta^*$ (it does not have zero elements, and there are no adjacent subnetworks that represent the exact same function).
It should be noted that $F_p^*$ only depends on $\theta^*$ and it is reasonable to find $\theta^*$ which satisfies the above requirements.
\end{remark}

\subsubsection{Assumption for Fisher residual}
The following condition is for eigenvalues of a matrix $\partial_\theta \partial_\theta^\top f_{\theta^*}(z) $.
This condition is designed in order to bound the difference $\partial_\theta \partial^\top M_n(\theta^*) - \hat{F}_{n,p}$, which is referred to as a \textit{Fisher residual}.
We define the following derivative
\begin{align*}
    w(z) =  {\frac{1}{f_{\theta^*}(z)} \partial_\theta \partial_{\theta}^\top f_{\theta^*}(z)},
\end{align*}
which satisfies $n^{-1} \sum_{i=1}^n w(z_i) = \partial_\theta \partial_\theta^\top M_n(\theta^*)  - \hat{F}_{n,p}$.
Further, we define a supremum of its norm $\nu_p^2 :=  \norm{\Ep[w(z)w(z)^\top]}_{op}$ and $\kappa_p := \esssup_{z \in \mZ} \norm{w(z)}_{op}$, and consider the following assumption:
\begin{assumption}[Bounded deviation of Fisher residual] \label{asmp:fisher_residual}
    The following conditions hold:
\begin{enumerate}
  \setlength{\parskip}{0cm}
  \setlength{\itemsep}{0cm}
        \item[(v)] $\nu_p^2 = o(p / \log p)$ as $p \to \infty$.
        \item[(vi)] $\kappa_p = o(p / \log p)$ as $p \to \infty$.
    \end{enumerate}
\end{assumption}
This assumption is not excessively restrictive.
If the matrices $\Ep[w(z)w(z)^\top]$ and $w(z)$ are completely dense and all the elements of the matrices of constant order, for example, then $\nu_p^2$ and $\kappa_p$ values are $O(p)$.
Hence, even a small degree of sparsity on the matrix $w$ can satisfy this assumption.
Technically, this assumption is utilized to applying the matrix Bernstein inequality to bound the Fisher residual.

\subsubsection{Assumption for Taylor residual}
We introduce an assumption for handling the Taylor residual term $R$, which contains the third order derivatives of $\log f_\theta$ at some $\underline{\theta} \in B_p(r,\hat{\theta})$.
To this aim, we define a $p \times p$ matrix by using the third-order derivatives for $j = 1,...,p$ and $z$:
\begin{align*}
    U^j_p(\underline{\theta},z) := \partial_{\theta_j} { \partial_\theta \partial_\theta^\top \log f_{{\underline{\theta}}}(z)},
\end{align*}
and define its upper bound and local Lipschitz constant
\begin{align*}
    \beta_p^j := \sup_{z \in \R^d}\sup_{\underline{\theta}\in \Theta} \|U^j_p(\underline{\theta},z)\|_{op}, \mbox{~and~}\
    \alpha_p^j := \sup_{z \in \R^d}\sup_{\underline{\theta}, \underline{\theta}' \in \Theta :\|{\underline{\theta} - \underline{\theta}'}\| < \delta, \underline{\theta} \neq \underline{\theta}'}  \frac{ \|{U^j_p(\underline{\theta},z) - U^j_p(\underline{\theta}',z)}\|_{op}}{\|{\underline{\theta} - \underline{\theta}'}\|},
\end{align*}
with some $\delta > 0$.
We introduce the following assumption on the decay speed of these terms:
\begin{assumption}[Shrinking third-order derivative] \label{asmp:taylor_residual}
    The following conditions hold:
\begin{enumerate}
  \setlength{\parskip}{0cm} 
  \setlength{\itemsep}{0cm} 
        \item[(vii)] $\sum_{j=1}^p (\beta_p^j)^2 = o(1)$ holds as $p \to \infty$. \label{asmp:weak_operator} 
        \item[(viii)] $\sum_{j=1}^p (\alpha_p^j)^2 = o(1)$ holds as $p \to \infty$.
    \end{enumerate}
\end{assumption}
This assumption describes the nonlinearity of models by handling higher-order interactions between parameters as $p$ increases.
For the simple linear regression with Gaussian noise, $\beta_j^p = \alpha^p_j = 0$ holds for any $j$ and $p$.
For nonlinear models, several cases can satisfy the condition.
First is sparsity; for example, Assumption \ref{asmp:weak_operator} is satisfied only if $o(1/\sqrt{p})$ elements of $\partial_\theta \partial_\theta^\top  M_n(\theta^*)$ are affected by $\theta_j$.
Another case is the group-wise interaction of parameters; parameters are divided into a few groups and interact with each other only within the groups.
In this case, $U_p^j$ becomes a block-wise sparse matrix and satisfies the assumption easily.
In Section \ref{sec:application}, we provide an explicit model for the case.

\subsubsection{Assumption for Conditional Independence}
We introduce an assumption for a conditional independence property of the Jacobian term $\partial_{\theta} M_n(\theta^*)$.
We define an $\R^p$-valued random vector
\begin{align*}
    J_i := \partial_\theta  \log f_{\theta^*}(z_i) = (\partial_{\theta_1} \log f_{\theta^*}(z_i), \dots, \partial_{\theta_p} \log f_{\theta}(z_i))^\top,
\end{align*}
for $i=1,...,n$.
It should be noted that $\Ep[J_i]  = 0$ holds for any $i=1,...,n$, and $J_i$ and $J_j$ are independent with $i \neq j$.
We consider the following assumption:
\begin{assumption}[Nearly Conditional Independence]\label{asmp:cross_variance}
    Define a matrix $S_{n,p} = (\hat{F}_{n,p} + \tau I_p)^{-1} F_p^* (\hat{F}_{n,p} + \tau I_p)^{-1}$.
    Then, the following relation holds:
    \begin{align}
        \frac{1}{n^2} \sum_{i,j=1, i \neq j}^n J_i^\top S_{n,p} J_j \to 0, \ \text{ in probability},\label{eq:cross0}
    \end{align}
    where the limit is taken over $p, n \to \infty$ and $p/n \to \gamma \in (0, \infty)$.
\end{assumption}
This assumption is satisfied when $S_{n,p}$ behaves as if it is independent of $J_i$ and $J_j$, and also $S_{n,p}$ has small variation.
This is because, in this scenario, the variance of the term in \eqref{eq:cross0} is asymptotically dominated by $(1/n^2)\tr(E[S_{n,p}^2])$, which goes to $0$ under a mild assumption. 
In the following Lemma \ref{lem:independence}, we formally provide a sufficient condition for the Assumption \ref{asmp:cross_variance} to hold.
\begin{lemma} \label{lem:independence}
    Suppose there exists a non-random matrix $L_{n, p}$ with $\lim_{p/n\to\gamma\in(0, \infty)} \tr(L_{n,p}^2) / p^2 = 0$, and it satisfies
    \begin{align}
        \mathbb{E}\qty[ J_i^\top S_{n,p} J_j J_u^\top S_{n,p} J_v  \middle| J_i, J_j, J_u, J_v] = J_i^\top L_{n,p} J_j J_u^\top L_{n,p} J_v  + U_{p, i, j}, \label{eq:lemma_covariance}
    \end{align}
    where $U_{p, i, j}$ satisfy $\lim_{p/n \to \gamma \in (0, \infty)} \sup_{i\neq j, u\neq v} \Ep[|U_{p, i, j}|] = 0$.
    Then, Assumption \ref{asmp:cross_variance} is satisfied.
\end{lemma}
We further provide numerical experiments to validate Assumption \ref{asmp:cross_variance} in Section \ref{sec:numerical}.



\subsection{Main Statement} \label{sec:main_statement}

We analyze the estimation risk of $\hat{\theta}$ using  the previous assumptions.
We firstly show that the variance term shrinks to zero as $n,p \to \infty$, and subsequently, we study the overall risk.
We provide an outline of the proof in Section \ref{sec:proof_outline}, and its full proof in the appendix.

Let  $\xi$ be a spectral measure of $\lim_{p \to \infty}(F_p^*)^{-1}$, which has a bounded support under Assumption \ref{asmp:basic}.
We define a weighted version of the Stieltjes transform as
\begin{align*}
    h^{(0)}_{\gamma,\tau}(a) := \int \frac{1}{(\tau / \gamma) \lambda - a} d\xi(\lambda).
\end{align*}
We also define $\lambda^* = \overline{\lambda}/\underline{\lambda}$ and a sequence $\{s_p\}_{p \in \N}$ of positive reals such as
\begin{align*}
    s_p := \frac{\nu_p^2 \log p}{p} \vee \frac{(\kappa_p \log p)^2}{p^2} \vee \sum_{j=1}^p (\alpha_p^j)^2 \vee \sum_{j=1}^p (\beta_p^j)^2.
\end{align*}
Note that under Assumption 2 and Assumption 3, $s_p \to 0$ as $p, n \to \infty$ and $p/n \to \gamma$.

Using the definitions above, we bound the asymptotic variance, derived with various approximation techniques for likelihood models and with a variant of the Marchenko-Pastur law \cite{pajor2009limiting}.
\begin{theorem}[Asymptotic Variance] \label{thm:variance}
    Suppose Assumptions \ref{asmp:basic}, \ref{asmp:fisher_residual}, \ref{asmp:taylor_residual} and \ref{asmp:cross_variance} hold.
    Let $\tau$ satisfy $\lim_{n,p \to \infty} \tau \to \overline{\tau}$ with $\overline{\tau} \geq 0$ and $\tau^2 \succ s_p$.
    Then, it almost surely satisfies
    \begin{align*}
        \limsup_{p,n \to \infty, p/n \to \gamma \in (0, \infty)}\|V_{n,p}(\hat{\theta})\|_{F_p^*}^2 \leq 2\left(1+\sqrt{\lambda^*}\right)^2 C_{\gamma, \overline{\tau}} \lim_{a \to 0} h_{\gamma,\overline{\tau}}(a),
    \end{align*}
    where $C_{\gamma, \overline{\tau}} = \lambda^*$ when $\gamma > 1$ and $\overline{\tau} = 0$, otherwise $C_{\gamma, \overline{\tau}} = 1$.
    $h_{\gamma,\overline{\tau}}$ is a function that satisfying
    \begin{align*}
        h_{\gamma,\overline{\tau}}(a) = h_{\gamma,\overline{\tau}}^{(0)} \left( a - \frac{1}{\gamma(1+h_{\gamma,\overline{\tau}}(a))}\right).
    \end{align*}
\end{theorem}
The extended Stieltjes transform, $\lim_{a \to 0} h_{\gamma,\overline{\tau}}(a)$, is a general term that can describe both the double descent and the regularized asymptotic risk curve.
Its details will be discussed in detail in Section \ref{sec:interpretation} and \ref{sec:interpretation_reguralized}.



Combined with the evaluation of the bias, we obtain the asymptotic bound for the overall estimation risk.
We recall that $r > 0$ is a radius of the parameter space defined in Assumption \ref{asmp:basic}.
\begin{theorem}[Asymptotic Estimation Risk] \label{thm:main}
    Suppose Assumption \ref{asmp:basic}, \ref{asmp:fisher_residual}, \ref{asmp:taylor_residual}, and \ref{asmp:cross_variance} hold.
    Let $\tau$ satisfy $\lim_{n,p \to \infty} \tau \to \overline{\tau}$ with $\overline{\tau} \geq 0$ and $\tau^2 \succ s_p$.
    Then, we almost surely have
    \begin{align*}
       \limsup_{p,n \to \infty, p/n \to \gamma \in (0, \infty)} \|\hat{\theta} - \theta^*\|_{F^*_p}^2
        &\leq 2\left(1+\sqrt{\lambda^*}\right)^2  \qty(C_{\gamma, \overline{\tau}} \lim_{a\to 0} h_{\gamma,\overline{\tau}}(a) + \overline{\lambda} r^2).
    \end{align*}
\end{theorem}
This result is obtained in a straightforward manner from Theorem \ref{thm:variance}, because it is a simple sum of the variance and the bias.
Although the bias does not vanish asymptotically, it is commonly found in asymptotic risk analyses (e.g., \cite{belkin2019two,hastie2019surprises}).
Importantly, the derived bound remains finite, even in situations where the number of parameters $p$ is considerably  larger than $n$.


Based on the results for the estimation risk, we can evaluate other types of asymptotic risks.
Here, we consider a nonparametric regression problem and determine the asymptotic risk of prediction.
Let $\mathcal{TN}_b(0, \sigma^2)$ be a truncated normal distribution with mean $0$ and variance $\sigma^2 > 0$ lying within the interval $(-b, b)$.
\begin{corollary}[Asymptotic Prediction Risk]\label{cor:reg}
    Consider the nonlinear regression model $y_i = g_{\theta^*}(x_i) + \epsilon_i$, where $\epsilon_i \sim \mathcal{TN}_b(0, \sigma^2)$ is the independent noise, and $g_\theta(x)$ is twice differentiable in $\theta$ in the neighbourhood of $\theta^*$.
    Suppose Assumptions \ref{asmp:basic}, \ref{asmp:fisher_residual}, \ref{asmp:taylor_residual}, and \ref{asmp:cross_variance} hold.
    Let $\tau$ satisfy $\lim_{n,p \to \infty} \tau \to \overline{\tau}$ with $\overline{\tau} \geq 0$ and $\tau^2 \succ s_p$.
    Additionally, we assume that $\sup_x \norm{\partial_{\theta} \partial_{\theta}^\top g_{\theta^*} (x)}_{op} = O_\Pp(1)$ holds.
    Then, there exists a constant $C_\sigma$ depending on only $\sigma^2$, and we obtain the following almost surely:
    \begin{align*}
        \limsup_{p,n \to \infty, p/n \to \gamma \in (0, \infty)} \norm{g_{\hat\theta} - g_{\theta^*}}_{L^2}^2
        &\leq C_\sigma \qty( \lim_{a\to 0} h_{\gamma,\overline{\tau}}(a) + r^2 + r^4).
    \end{align*}
\end{corollary}
We find that this bound for the prediction risk is almost similar to that for the estimation risk.

The term $\lim_{a \to 0} h_{\gamma,\overline{\tau}}(a)$ is a general expression that can describe the two asymptotic risks with the double descent and the regularized estimation.
In the following subsections, we will present specific calculations and interpretations for the two cases.

\subsubsection{Double Descent Case ($\overline{\tau} = 0$)} \label{sec:interpretation}
When the limit regularization coefficient $\Bar{\tau}$ is zero, the term describes the double descent curve as the result shown in Figure \ref{fig:risk_intro}.



\begin{proposition}[Double Descent Variance] \label{prop:double_descent}
    Suppose that Theorem \ref{thm:variance} holds with the setting $\overline{\tau} = 0$.
    Then, for $\gamma \neq 1$, we obtain
    \begin{align*}
        \lim_{a \to 0} h_{\gamma,0}(a) =  \frac{\gamma\mone\{\gamma < 1\}}{1 - \gamma} + \frac{ \mone\{\gamma > 1\}}{\gamma - 1} .
    \end{align*}
\end{proposition}
This result shows that any nonlinear or deep model can achieve double descent, as long as it satisfies our setting and assumptions.
The explicit form corresponds exactly to the double descent risk curve for linear regression in Theorem 2 and 3 in \cite{hastie2019surprises}. 
Interestingly, even for complex nonlinear deep models, the double descent curve is the same as for linear models.

Substituting Proposition \ref{prop:double_descent} into Theorem \ref{thm:variance} yields the following result, whose proof has been excluded here.
We define $C_* := 2\left(1+\sqrt{\lambda^*}\right)^2 C_{\gamma, \overline{\tau}}$.
\begin{corollary}[Double Descent Estimation Risk]
    If the setting in Proposition \ref{prop:double_descent} holds, then
    \begin{align*}
       \limsup_{p,n \to \infty, p/n \to \gamma \in (0,\infty)} \|\hat{\theta} - \theta^*\|_{F^*_p}^2
        &\leq C_*\left\{ \left( \frac{\gamma\mone\{\gamma < 1\}}{1 - \gamma} + \frac{  \mone\{\gamma > 1\}}{\gamma - 1} \right) + \overline{\lambda} r^2\right\}.
    \end{align*}
\end{corollary}

\subsubsection{Asymptotically Regularized Case  ($\overline{\tau} > 0$)} \label{sec:interpretation_reguralized}

When we set $\overline{\tau} > 0$, we obtain the regularized asymptotic risk curve.
In the following, we consider the case where the limit $\lim_{p \to \infty} F_p^*$ has a pointwise spectrum such that $\xi$ is a Dirac measure.
This setting, which includes linear regression under isotropic covariates, is addressed \cite{hastie2019surprises}.
\begin{example}[Point Spectrum] \label{example:point_spectrum}
    Consider the setting in Theorem \ref{thm:variance} and set $\overline{\tau} > 0$.
    Suppose that $\xi$ is a Dirac measure at $\check{\lambda} > 0$.
    Then, we obtain
    \begin{align*}
        \lim_{a \to 0} h_{\gamma,\overline{\tau}}(a) &= \frac{-(\check{\lambda}\overline{\tau} + 1 - \gamma) + \sqrt{(\check{\lambda}\overline{\tau} + 1 - \gamma)^2 + 4 \check{\lambda}\overline{\tau} \gamma}}{2 \check{\lambda}\overline{\tau}}.
    \end{align*}
\end{example}
This risk bound is similar to the asymptotic risk for existing linear ridge regressions \cite{dobriban2018high,hastie2019surprises}.
Figure \ref{fig:reguralized} illustrates the asymptotic risk curve obtained from the result and also from the previous study \cite{hastie2019surprises}.
It should be noted that our results do not require the linearity of the model.

Our theorem can also deal with different spectral measures of $\xi$.
We pick a uniform measure and a semi-circular measure as examples and illustrate its extended Stieltjes transform in Figure \ref{fig:reguralized2}.
We can observe that even with different measurements, the rick curves are roughly similar.
We present the detailed derivation process in Section \ref{sec:stieltjes}.

\begin{figure}[htbp]
    \centering
    \begin{minipage}{0.47\hsize}
    \captionsetup{width=0.9\hsize}
    \includegraphics[width=\hsize]{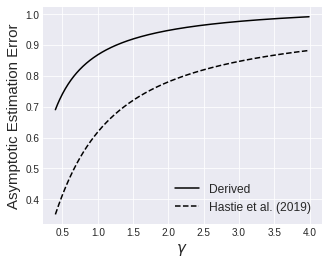}
    \caption{Derived Asymptotic Estimation Risk Bound (solid). The dashed curve depicts the error presented in \cite{hastie2019surprises} for comparison. \cite{hastie2019surprises} for comparison. \label{fig:reguralized}}
    \end{minipage}
    \begin{minipage}{0.47\hsize}
    \captionsetup{width=0.9\hsize}
    \includegraphics[width=\hsize]{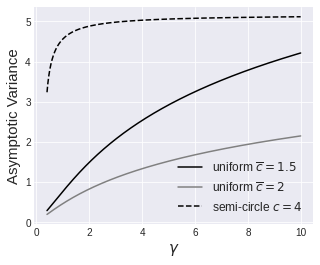}
    \caption{Values of $\lim_{a \to 0} h_{\gamma,\overline{\tau}}(a)$ when $\xi$ is a uniform measure on $[1, \overline{c}]$ or a semi-circular measure with center $c$.\label{fig:reguralized2}}
    \end{minipage}
\end{figure}

\subsection{Proof Outline of Theorem \ref{thm:variance} and \ref{thm:main}} \label{sec:proof_outline}



\subsubsection{Basic Decomposition}

We begin with a basic decomposition of the estimation risk and divide the risk into two parts: the principal term, which will be analyzed using the random matrix theory, and other negligible terms.
We provide the basic decomposition of $\theta^* - \hat{\theta}$ by continuing the representation in \eqref{def:bias_var} as
\begin{align}
     \theta^* - \hat\theta  &= \underbrace{ (\partial_\theta \partial_\theta^\top M_n(\theta^*) + \tau I_p)^{-1} (\hat{F}_{n,p} + \tau I_p)}_{=:V_0} (\hat{F}_{n,p} + \tau I_p)^{-1}(\partial_\theta M_n(\theta^*) + R + \tau \theta^*)\notag \\
     &= V_0 \Biggl\{\underbrace{ (\hat{F}_{n,p} + \tau I_p)^{-1}\partial_\theta M_n(\theta^*)}_{=: V_1} +  \underbrace{(\hat{F}_{n,p} + \tau I_p)^{-1}R}_{=:V_2} + \underbrace{(\hat{F}_{n,p} + \tau I_p)^{-1}\tau \theta^*}_{=: B_0} \Biggr\}. \label{eq:mle_diff}
\end{align}
The term $V_0$ represents the difference between the Hesse matrix and the empirical Fisher information matrix and is related to the Fisher residual.
The principal term $V_1$ plays a central role in the asymptotic risk.
Its asymptotic behavior is investigated using the random matrix theory.
The scaled Taylor residual $V_2$ is an important representation of the nonlinearity of the likelihood model $f_\theta$.
The scaled bias $B_0$ corresponds to a bias term associated with the regularization.
In the following parts, each of the terms will be evaluated separately.

\subsubsection{Fisher Residual Related Term $V_0$}
We demonstrate $V_0$ converges to $1$ in terms of the operator norm with Assumption \ref{asmp:fisher_residual}.
We utilize the representation of {Fisher residual} $ \hat W_{n, p} := \partial_\theta \partial_\theta^\top M_n(\theta^*) - \hat F_{n,p} = n^{-1} \sum_{i=1}^m w(z_i) $.
Considering the well-known property of derivatives of loglikelihood, $\Ep[\partial_\theta \partial_\theta^\top M_n(\theta^*)] = \Ep[\hat F_{n,p}]$ holds (Lemma 5.3 in \cite{lehmann2006theory}); hence, $\Ep[\Hat{W}_{n,p}] =  0$ holds.

We bound $\|V_0\|_{op}$, which is necessary to bound $\|\theta^* - \hat{\theta}\|_{F_p^*}^2$.
A simple calculation yields
\begin{align*}
    \|V_0\|_{op} &\leq 1 + \norm{{F^*_p}^{1/2} \qty( \hat{F}_{n,p} - \hat W_{n, p} + \tau I_p)^{-1} \hat W_{n, p} {F^*_p}^{-1/2}}_{op}.
\end{align*}
To bound the norm, we have to bound $\|{\hat W_{n, p}}\|_{op}$ and show that $\hat{F}_{n,p} - \hat W_{n, p} + \tau I_p$ is a positive definite matrix with high probability.
To achieve this, we apply the matrix Bernstein inequality \cite{tropp2015introduction} and obtain
\begin{align*}
    \Pp\qty(\|{\hat W_{n, p}}\|_{op} \geq t) \leq 2p \exp\qty( \frac{-n t^2 / 2}{\nu_p^2 + \kappa_p t / 3} ),
\end{align*}
for any $t > 0$.
For a choice $t = \tau/2$ and from the assumption $\tau^2 \succ s_p$, we reach the aim and verity that $\|V_0\|_{op} =O(1)$ holds.

\subsubsection{Principal Term $V_1$}
We decompose the principal term $V_1$ and demonstrate that this term mainly determines the asymptotic variance.
We had defined a $p \times n$ random matrix $\hat{J}_{n,p} := (J_1,...,J_n) =  (\partial_\theta \log f_{\theta^*}(z_1), \dots, \partial_\theta \log f_{\theta^*}(z_n))$.
Using the identity $\|b\|_{F_p^*}^2 = \trace(F_p^{*1/2} b b^\top F_p^{*1/2})$ for a vector $b \in \R^p$, we decompose $\|V_1\|_{F_p^*}^2$ as
\begin{align*}
    \|V_1\|_{F_p^*}^2 &=  \tr\qty({F^*_p}^{1/2} ( \hat F_{n, p} + \tau I_p)^{-1}\partial_\theta M_n(\theta^*) \partial_\theta M_n(\theta^*)^\top ( \hat F_{n, p} + \tau I_p)^{-1} {F^*_p}^{1/2} ) \\
     &=\underbrace{\tr\qty( ( \hat F_{n, p} + \tau I_p)^{-1} {F^*_p} ( \hat F_{n, p} + \tau I_p)^{-1} n^{-2} \hat J_{n, p} \hat J_{n, p}^\top  ) }_{=:V_{1,1}}\\
     & \quad + \underbrace{\tr\qty( ( \hat F_{n, p} + \tau I_p)^{-1} {F^*_p} ( \hat F_{n, p} + \tau I_p)^{-1}(\partial_\theta M_n(\theta^*) \partial_\theta M_n(\theta^*)^\top - n^{-2} \hat J_{n, p} \hat J_{n, p}^\top)  )}_{=: V_{1,2}}.
\end{align*}
The term $V_{1,1}$ represents a diagonal effect of a quadratic term $\partial_\theta M_n(\theta^*) \partial_\theta M_n(\theta^*)^\top$, because the term is a weighted version of $\trace(\hat{J}_{n,p} \hat{J}_{n,p}^\top) =  \sum_{k=1}^p \sum_{i=1}^n (\partial_{\theta_k} \log f_{\theta^*}( z_i))^2$.
The second term $V_{1,2}$ is interpreted as an off-diagonal effect, because it is a weighted version of  $\partial_\theta M_n(\theta^*) \partial_\theta M_n(\theta^*)^\top -  n^{-2} \hat{J}_{n,p} \hat{J}_{n,p}^\top = n^{-2} \sum_{k=1}^p \sum_{i,j=1, i\neq j}^n \partial_{\theta_k} \log f_{\theta^*}( z_i)\partial_{\theta_k} \log f_{\theta^*}( z_j)$.
We find that $V_{1,2}$ is asymptotically negligible by Assumption \ref{asmp:cross_variance}.

The term $V_{1,1}$ is a key term of the asymptotic variance.
We define a normalized random matrix $\Tilde{J}_{n,p} := (1/\sqrt{p}) (F^*_p)^{-1/2} \hat{J}_{n,p}$ whose column includes an identity covariance matrix.
$V_{1,1}$ is rewritten as follows such that its limit will be achieved, which is suitable for the random matrix theory:
\begin{align*}
    V_{1,1} \leq  \trace\left((\Tilde{J}_{n,p} \Tilde{J}_{n,p}^\top + (n/p) \tau {F^*_p}^{-1})^{-1}\right) \to \lim_{z \to 0} h_{\gamma,\overline{\tau}}(z), ~ (n,p \to \infty, p/n \to \gamma),
\end{align*}
where $h_{\gamma,\overline{\tau}}(z)$ is defined in the statement of Theorem \ref{thm:variance}. 
The limit follows from the extension of the Marchenko-Pastur law in Lemma \ref{lem:marchenko}, which is presented below.
\begin{lemma}[Theorem 3.3 in \cite{pajor2009limiting}]\label{lem:marchenko}
    Let $T_{n,p} = (t_1, \dots, t_n)$ be a $p \times n$ random matrix whose columns are independent and identically distributed. 
    Assume that $t_i$ has a log-concave density function, and $\mathbb{E}[t_1] = 0$ and $\Cov(t_1) = I_p$ hold.
    With $p \times p$ non-random matrix $E_p$, let us consider a $p \times p$ random matrix $Q_{n,p} = T_{n,p}T_{n,p}^\top + E_p$, and its normalized spectral measure $\mu_{n,p}$.
    Then, there exists a non-random measure $\mu$ such that for any interval $\Delta \subset \R$,
    we obtain the following in probability
    \begin{equation*}
        \lim_{n,p \to \infty, n/p\to c \in [0, \infty)} \mu_{n, p} (\Delta) = \mu(\Delta).
    \end{equation*}
    Further, its Stieltjes transform $f(z) := \int_{\R} \frac{\dd{\mu(\lambda)}}{\lambda - z}$ is uniquely determined as
    \begin{align*}
        f(z) &= f^{(0)}\qty(z - \frac{c}{1 + f(z)}), \mbox{~and~}f^{(0)}(z) = \int_{\R} \frac{\dd{\nu(\lambda)}}{\lambda - z},
    \end{align*}
    where $\nu$ is a limit of the normalized spectral measure of $E_p$.
\end{lemma}

\subsubsection{Weighted Taylor Residual $V_2$}
We claim that a norm of $V_2$ is negligible from Assumption \ref{asmp:taylor_residual} for the residual term $R$ in the Taylor expansion.
We consider a uniform upper bound for a $j$-th element $R_j$ of $R$ as
\begin{align*}
    R_j \leq  \sup_{\theta, \theta' \in \Theta} (\theta - \theta^*)^\top U^j_p(\theta',z) (\theta - \theta^*),
\end{align*}
and evaluate the bounds using several empirical process techniques, such as the Rademacher complexity and the Dudley integral, associated with the concentration inequalities (for an overview of the techniques, see \cite{gine2016mathematical}).
Consequently, we obtain the following results.
\begin{proposition} \label{prop:taylor_residual}
    For any $\tau > 0$ and sufficiently large $p$, we obtain the following bound with probability at least $1 - 2/p$:
    \begin{align*}
        \|V_2\|_{F_p^*}^2 = O\qty( \frac{\lambda_{\max}(F^*_p)}{\tau^2} \left\{ r^6 \sum_{j=1}^p \qty(\beta_p^j)^2 \vee r^4 \sum_{j=1}^p \qty(\alpha_p^j)^2 \right\} ).
    \end{align*}
\end{proposition}
By combining this result with Assumption \ref{asmp:taylor_residual}, we can state that $\|V_2\|_{F_p^*}^2 = o(1)$ with high probability holds.
The formal statement is provided in Lemma \ref{lem:bound_R} in the appendix.

\subsubsection{Weighted Bias $B_0$}
We do not discuss the term $B_0$ extensively, because the bias is not the main focus of this study, and as it has already been addressed by previous studies, e.g., \cite{wu2020optimal}.
We simply bound it using Assumption \ref{asmp:basic} and the setting of $\tau$ and then obtain $\|B_0\|_{F_p^*}^2 \lesssim r^2$.

\section{Asymptotic Risks of Specific Models} \label{sec:application}

We analyze several specific models and derive the asymptotic risks by utilizing the derived results as an illustration.
Specifically, we evaluate whether the assumptions are satisfied by several models.

\subsection{Linear Regression Models}

We investigate a simple linear regression model.
We consider an i.i.d. sequence of pairs $(x_1, y_1), \dots, (x_n, y_n) \in \R^{p} \times \R$. Write $x_i = (x_{i1}, \dots, x_{ip})$.
Consider the following regression model with a parameter $\theta \in B_p(0,r)$ as
\begin{align}
  y_i &= \theta^\top x_i + \epsilon_i, \label{model:reg}
\end{align}
with independent noise $\epsilon_i \sim \mathcal{TN}_b(0, \sigma^2)$ and covariates $x_{i} \sim P_X$, where $P_X$ has a log-concave density supported on $B_p(0, s)$ for some $s > 0$.

We discuss the assumptions in Section \ref{sec:assumption} in this regression setting.
In Assumption \ref{asmp:basic}, $\partial_\theta \log f(z) = \epsilon_i x_i$, which has log-concave density owing to the log-concavity of $P_X$.
The other conditions in Assumption \ref{asmp:basic} are trivially satisfied.
Assumption \ref{asmp:fisher_residual} is also satisfied, as we obtain
\begin{align*}
    \norm{w(z_i)}_{op} \leq ((b^2 + 1)/\sigma^2) \|{x_i x_i^\top}\|_{op} \leq ((b^2 + 1)/\sigma^2)s^2 = O(1),
\end{align*}
which implies $\norm{\Ep[w(z_i) w(z_i)^\top]}_{op} = O(1)$.
Assumption \ref{asmp:taylor_residual} also holds, as the Taylor residual is $R = 0$ in this setting because the second and third order derivatives of $\theta^\top x$ with respect to $\theta$ are $0$. 
We validate Assumption \ref{asmp:cross_variance} numerically, which is deferred to the subsection \ref{sec:numerical}.

\subsection{Additive Regression Model and Parallel Neural Networks} \label{sec:additive}
We consider a general nonparametric regression problem and show that a certain class of deep neural networks satisfy the assumptions in Section \ref{sec:assumption}.
Specifically, we introduce the additive structure into a regression model in order to fulfill Assumption \ref{asmp:fisher_residual} and \ref{asmp:taylor_residual}, then apply it to the design of deep neural networks.

We consider an i.i.d. sequence of pairs $(x_1, y_1), \dots, (x_n, y_n) \in \R^{d} \times \R$ generated from
the following regression model
\begin{align}\label{mod:parallel}
  y_i &= g_{\theta^*}(x_i) + \epsilon_i.
\end{align}
where $g_{\theta}(x)$ is a regression model with true parameter $\theta = \theta^* \in B_p(0,r)$, such as $\sup_\theta \norm{g_\theta}_{L^\infty} < b'$. 
$x_i \sim P_X$ is the i.i.d. covariate from a probability measure $P_X$ on $\R^d$ with log-concave density.
$\epsilon_i$ is the i.i.d. noise with density
    $p(\epsilon) = C_1 \exp\qty{- C_2 \epsilon^2} \1_{[-b, b]}(\epsilon)$
for some known constants $C_1 , C_2  , b > 0$, and independent to $x_i$ for any $i$.
We estimate the parameter $\theta$ through the empirical risk minimization with a ridge penalty:
\begin{align}
    \hat\theta = \argmin_\theta \frac{1}{n} \sum_{i=1}^n ( y_i - g_{\theta} (x_i) )^2 + \tau \norm{\theta}^2. \label{eq:optim_regression}
\end{align}
This minimization problem is equivalent to the maximization of likelihood $f_\theta(z) \propto \mathrm{exp} (- C_2 \abs{y - g_\theta(x)}^2)$.

We specify the additive structure of $g_{\theta^*}$ by defining a $M$ partition of $p$ parameters.
Let $\Delta_1, \dots, \Delta_M$ be a partition of $[p] := \{1,2,...,p\}$ such that $\bigcup_{m=1}^M \Delta_m = [p]$ and $\Delta_{m} \cap \Delta_{m'} = \emptyset$ for $m \neq m'$. 
$\card(\Delta_m) \leq p/M + 1$ holds for any $m = 1,...,M$.
Let $g_{\Delta_m(\theta)}$ be a function that only depends on a sub parameter $\theta_{\Delta_m} = (\theta_\ell)_{\ell \in \Delta_m}$. 
Without loss of generality, we set
$j < j'$ for $j \in \Delta_m$ and $j' \in \Delta_{m'}$ if $m < m'$ hold.
We consider the following additive model:
\begin{align}\label{asmp:disjoint_g}
    g_{\theta}(x) = \frac{1}{M} \sum_{m=1}^M g_{\Delta_m(\theta)}(x).
\end{align}
These additive nonparametric regression models have been actively investigated \cite{stone1985additive,hastie1987generalized}.
This division by $M$ is introduced to keep the scale of each sub-model $g_{\Delta_m}(x)$ constant.

We introduce several settings for the nonlinear model; specifically, the boundedness and Lipschitz continuity are imposed on derivatives.
For $b = 1,2,3$, let $j_{1:b} = (j_1,...,j_b) \in \Delta_m^{b} := \times_{k=1}^b \Delta_m$ be a tuple of $b$ indexes from $\Delta_m$.
We define the derivative as $\partial_\theta^{j_{1:b}} := \partial_{\theta_{j_1}}\cdots \partial_{\theta_{j_b}} g_{\Delta_m}$.
For any $b = 1,2,3$, we assume that there exists a finite constant $c > 0$, such that
\begin{align}
    \max_{m \in [M]}\max_{j_{1:b} \in \Delta_m^b} \left\{\sup_{\theta' \in \Theta} \| \partial_\theta^{j_{1:b}} g_{\Delta_m(\theta')}\|_{L^\infty} \vee \sup_{\theta', \theta'' \in \Theta} \frac{\|{\partial_{\theta}^{j_{1:b}} g_{\Delta_m(\theta')} - \partial_{\theta}^{j_{1:b}} g_{\Delta_m(\theta'')}}\|_{L^\infty}}{\|\theta'_{\Delta_m} - \theta''_{\Delta_m}\|} \right\}  \leq c. \label{asmp:additive_reg}
\end{align}
These derivatives are bounded and Lipschitz continuous in terms of the sub-parameters from $\Delta_m$.
A rigorous notation will be provided in Section \ref{sec:appendix_appli}.
Then, we obtain the following result.
\begin{proposition}\label{prop:additive}
    Consider the regression problem \eqref{mod:parallel} with the model \eqref{asmp:disjoint_g}.
    Suppose Assumption \eqref{asmp:additive_reg} holds and set $M \succ p^{3/4}$.
    Then, the followings hold as $p \to \infty$:
    \begin{align*}
        \sum_{j=1}^p \qty(\beta_p^j)^2 = o(1),
        \ \sum_{j=1}^p \qty(\alpha_p^j)^2 = o(1), \ \nu_p^2 &= o(1), \mbox{~and~}\ \kappa_p = o(1).
    \end{align*}
\end{proposition}
This result indicates that the additive structure satisfies the regularity conditions, especially, Assumption \ref{asmp:taylor_residual} on the third derivative matrix $U_p^j$.
With the additive model \eqref{asmp:disjoint_g}, the interactions between the parameters are limited within each model and there are no interactions with other models.
That is, we obtain a $(k,\ell)$-th element $ [U_p^j]_{k,\ell}$ of $U_p^j$ such as
\begin{align*}
    [U_p^j]_{k,\ell} = 
    \begin{cases}
        \partial_{\theta_j}\partial_{\theta_k}\partial_{\theta_\ell} \Ep[M_n(\theta)] & \mbox{~if~} \exists m \in [M], \mbox{~s.t.~}j,k,\ell \in \Delta_m,\\
        0 & \mbox{~otherwise}.
    \end{cases}
\end{align*}
Hence, the matrix $U_p^j$ becomes sparse and satisfies Assumption \ref{asmp:taylor_residual} easily.

We can apply the result to deep neural networks.
The additive nonparametric regression form is equivalent to a parallel deep neural network, as illustrated in Figure \ref{fig:para_network}, and also there is no form constraint on the functions $g_{\Delta_m},m=1,..,M$ by the sub-networks.
Hence, our asymptotic risk analysis can be applied to deep neural networks when they are parallelized.

\subsection{Ensemble learning}

We study the ensemble learning scheme for the regression problem, which develops a predictive model by combining different weak learners.

We consider the following setting, which utilizes the formulation of the previous section.
Assume the data generating process in Section \ref{sec:additive} and also the data are generated by the additive model in \eqref{mod:parallel}.
To develop a predictive model, we define $M$ weak learners $g_{\theta_1}(x),...,g_{\theta_M}(x)$, where $\theta_m \in \Theta_m$ is a $p_m$-dimensional parameter vector for the $m$-th learner $g_{\theta_m}(x)$, with its parameter space $\Theta_m$.
It should be noted that there are $p = \sum_{m=1}^M p_m$ parameters in total.
We assume that the weak learners are bounded and Lipschitz continuous, similar to the constraint in \eqref{asmp:disjoint_g}.
Analogous to the optimization problem in \eqref{eq:optim_regression}, we can consider 
the following penalized maximum likelihood estimator:
\begin{align}
    \argmin_{\theta_1 \in \Theta_1, \dots, \theta_M \in \Theta_M} \frac{1}{n} \sum_{i=1}^n \abs{\frac{1}{M} \sum_{m=1}^M (y_i - g_{\theta_m}(x_i))}^2 + \tau \sum_{m=1}^M \norm{\theta_m}^2. \label{opt:ensemble}
\end{align}

We verify that Assumptions \ref{asmp:fisher_residual} and \ref{asmp:taylor_residual} hold with the learning problem.
Let $g_\theta(\cdot) = M^{-1} \sum_{m=1}^M g_m(\cdot)$ be the aggregation of weak learners.
The setting is almost the same as Proposition \ref{prop:additive}, we provide the following result without proof.
\begin{proposition}\label{prop:ensemble}
    Consider the optimization problem in \eqref{opt:ensemble} with the weak learners $g_1(x),...,g_M(x)$.
    Suppose the learners satisfy the property in \eqref{asmp:additive_reg} and set $M / p_m^{3} \to \infty$ for all $m = 1,...,M$ as $M \to \infty$.
    Then, a predictive model $g_\theta(\cdot) = M^{-1} \sum_{m=1}^M g_m(\cdot)$ satisfies Assumptions \ref{asmp:fisher_residual} and \ref{asmp:taylor_residual} as $M \to \infty$.
\end{proposition}
In this result, no constraints are imposed on the form of each learner, except for the boundedness and Lipschitz continuity; hence, many different models are available, including deep models.
One important condition is that the number of parameters $p_m$ for each learner does not increase excessively with the number of learners $M$.

\subsection{Residual Network and Minimax Risk}

We analyze a minimax prediction risk of a neural network called \textit{Residual Network} (ResNet) \cite{he2016deep} by applying the results on parallel deep neural networks. 
ResNet has a specific structure called skip connections, which can increases the number of layers easily and is known for its ability to make accurate predictions.

We define a ResNet model as follows.
Let $L, M \in \N$ be hyper-parameters and $d \in \N$ is a width of ResNet which is fixed for brevity. 
For $m =1,...,M$ and $\ell = 1,...,L$, we define a linear map $\rho_{m,\ell}: \R^d \to \R^d$ with a parameter vector $\eta_{m,\ell}$.
Then, with an activation function $\sigma: \R^d \to \R^d$, we define a function by a deep neural network with $L$ layers as
\begin{align*}
    \rho_{\Delta_m}(x) :=  \sigma \circ \rho_{m,L} \circ \sigma  \circ \rho_{m,L-1} \circ \cdots \circ \sigma \circ \rho_{m,1}(x).
\end{align*}
for $m = 1,...,M$.
We define a ResNet model $g_\eta: \R^d \to \R$ with an identity map $\mathrm{id}(\cdot)$ and an additional parameter vector $\eta_0 \in \R^{d}$  as
\begin{align*}
    \rho_\eta(x) := \eta_{0}^\top(\rho_{\Delta_M} + \mathrm{id}) \circ (\rho_{\Delta_{M-1}} + \mathrm{id}) \circ \cdots \circ (\rho_{\Delta_1} + \mathrm{id}) (x).
\end{align*}
Here, $\eta := \{(\eta_{m,\ell})_{m,\ell = 1}^{M,L},\eta_0 \}$ denotes a tuple of all the parameters.
The coordinate $(\rho_{\Delta_m} + \mathrm{id})$ is called a residual block and $M$ denotes a number of residual blocks.
The identity maps are referred to as skip connections. 

We discuss an asymptotic minimax predictive risk of ResNet under the regression setting \eqref{mod:parallel}.
A key fact by \cite{oono2019approximation} is that ResNet can realize parallel deep neural networks, which is also referred to a block sparse model, when $g_{m,\ell}$ is a convolution layer.
With the convolution setting, let $\mG^{\text{(PNN)}}_{M,p}$ be a class of parallel neural networks with $M$ sub-models and $p$ parameters as Section \ref{sec:additive}, and $\mG^{\text{(RN)}}_{M,q}$ be a class of ResNet with $M$ blocks and $q$ parameters.
For minimax risk analysis, we define observed data $D_n := \{(x_1,y_1),...,(x_n,y_n)\}$, and estimators $\hat \rho = \hat \rho(D_n) \in \mG^{\text{(RN)}}_{M,q}$ and $\hat g = \hat g(D_n) \in \mG^{\text{(PNN)}}_{M,p}$ for each model, which maps the observed data to a function by the model.

According to Theorem 5 in \cite{oono2019approximation}, $\mG^{\text{(PNN)}}_{M,p} \subset \mG^{\text{(RN)}}_{M,q}$ holds with a relation $q = p + p'$ for $p' = O(M)$.  
Hence, we obtain the following inequality with the estimator $g_{\hat{\theta}}$ from \eqref{eq:optim_regression} with a parallel neural network:
\begin{align}
    \inf_{\hat \rho\in \mG^{\text{(RN)}}_{M,q}} \| \hat \rho - g_{\theta^*} \|^2_{L^2} \leq \inf_{\hat g \in \mG^{\text{(PNN)}}_{M,p}} \| \hat g - g_{\theta^*} \|^2_{L^2} \leq \|{ g_{\hat{\theta}} - g_{\theta^*}}\|^2_{L^2}, \label{ineq:minimax}
\end{align}
for any $n,p,M$ and $q = p + p'$. 
Here, the infimums are taken from any measurable estimators $\hat \rho$ and $\hat g$.
Based on the inequality, we obtain the following result:
\begin{corollary}\label{cor:resnet}
    Consider the regression problem \eqref{mod:parallel}  and $\mG^{\mathrm{(RN)}}_{M,q}$ whose sub-models $\{\rho_{\Delta_m}\}_{m=1}^M $ are bounded and Lipschitz continuous as \eqref{asmp:additive_reg}.
    Suppose Assumption \ref{asmp:basic}, \ref{asmp:cross_variance}, and $\sup_x \norm{\partial_{\theta} \partial_{\theta}^\top g_{\theta^*} (x)}_{op} = O_\Pp(1)$ hold.
    Then, for $M \succ q^{3/4}$, we obtain
\begin{align*}
    \limsup_{q,n \to \infty, q/n \to \gamma} \inf_{\hat \rho \in \mG^{\mathrm{(RN)}}_{M,q}} \norm{\hat \rho - g_{\theta^*}}^2_{L^2} 
    \lesssim \lim_{a \to 0} h_{\gamma,  \overline{\tau}}(a) + r^2 + r^4,
\end{align*}
for any $\Bar{\tau} \geq 0$. 
\end{corollary}
This result is straightforward from \eqref{ineq:minimax}, Corollary \ref{cor:reg}, and Proposition \ref{prop:additive}, hence its proof is omitted.
Since we bound the minimax risk with ResNet by that of parallel deep neural networks as a proxy, we achieve the asumptotic minimax risk without assuming Assumption \ref{asmp:fisher_residual} and \ref{asmp:taylor_residual}.


\section{Experimental Study for Parallel Deep Neural Network}

We verify our theoretical findings by conducting real data experiments. 
Specifically, we test the validity of the result on parallel deep neural networks in Section \ref{sec:application} by investigating experimental error variances.

We consider a classification problem with the CIFAR 10 dataset \cite{krizhevsky2009learning}, which contains pairs of an input image and an output label from $10$ classes.
From the dataset, we utilize $n = 25,000$ images for training and $10,000$ images for testing.
We design the following architecture of deep neural networks for the classification problem.
For a non-parallel deep neural network, we utilize the ResNet \cite{he2016deep} for images with $18$ layers.
For a parallel deep neural network, we set $M=3$ and build a new neural network with three non-parallel networks side by side.
To control the number of parameters of both non-parallel and parallel networks, we vary the width (the number of filters) of a convolution layer in the ResNet from $1,2,...$ to $64$.
We set the loss function as the cross-entropy loss, that is, we consider the following negative log-likelihood:
\begin{align*}
    M_n(\theta) = - \sum_{i=1}^n \sum_{k=1}^K \mone\{y_i = k\} \log( g_\theta(x_i)_k),
\end{align*}
with an observed $K$-label $y_i \in \{1,2,...,K\}$, and a $[0,1]^K$-valued normalized model $g_\theta(\cdot) = (g_\theta(\cdot)_1,...,g_\theta(\cdot)_K)$ such that $\sum_{k=1}^K g_\theta(\cdot)_k = 1$.
With the neural networks as a model and the loss, we solve the parameter optimization problem by using the stochastic gradient descent with a momentum of $0.9$ and its learning rate of $0.1$.
We further apply a penalization with the coefficient of $0.0005$.
We replicate the procedure for $10$ times and report the mean and standard deviation.
The settings mainly follows those of \cite{yang2020rethinking}.

To study the asymptotic behavior, we report the variance of the asymptotic risk of the non-parallel/parallel deep neural networks.
Because it is difficult to track the values of the parameters in neural networks, we analyze the loss function instead.
Hence, we apply the bias-variance decomposition of the cross-entropy loss developed by \cite{yang2020rethinking} and calculate a variance term of the loss using the images for the test.

\begin{figure}[htbp]
    \centering
    \begin{minipage}{0.47\hsize}
    \centering
    \captionsetup{width=0.9\hsize}
    \includegraphics[width=0.85\hsize]{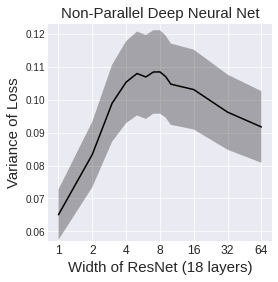}
    \end{minipage}
    \begin{minipage}{0.47\hsize}
    \centering
    \captionsetup{width=0.9\hsize}
    \includegraphics[width=0.85\hsize]{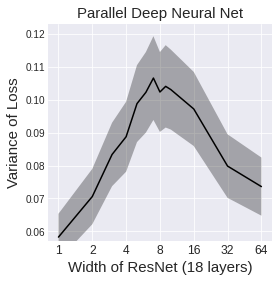}
    \end{minipage}
    \caption{Variance of the cross-entropy loss against the width of the ResNet with $18$ layers as the model size. 
    The values are calculated using the CIFAR 10 dataset with the non-parallel (left) and the parallel (right) deep neural networks.  \label{fig:simulation}}
\end{figure}

Figure \ref{fig:para_network} presents the mean and standard deviation of the variance of the loss with the non-parallel (left) and parallel (right) deep neural networks.
The horizontal axis shows a logarithm of the width of the ResNet with $18$ layers, and the vertical axis shows the loss variance.
The grey region presents the standard deviation.
From this result, we can deduce the following: (i) In both cases, the variance rises and then falls. 
In other words, a phenomenon similar to double descent is observed.
(ii) In both cases, the descent begins at a common location, with a width of $8$.
(iii) The behavior of the parallel network is more similar to that of a double descent.
These results are in line with our main theorems stating that general nonlinear models, such as deep neural networks, also exhibit the double descent and bounded asymptotic risk.
Further, they support our theoretical finding that parallelization is a sufficient condition for double descent.

\section{Discussion}

\subsection{Interpretation of Results}
In this section, we provide additional discussion regarding our results.

The penalty term $\tau \|\theta\|^2$ for the estimator is not always a regularization for the asymptotic risk. 
Whether the asymptotic variance follows the double descent phenomena depends solely on its limit $\overline{\tau} = \lim_{n,p\to \infty} \tau$.
If $\overline{\tau} = 0$, then the asymptotic variance is not regularized and double descent occurs. 
Even in our proof, the non-asymptotic $\tau$ is only used to guarantee the positive definiteness of the empirical Fisher information matrix; hence, the non-asymptotic value of $\tau$ does not affect the asymptotic variance.

Among the assumptions in Section \ref{sec:assumption}, Assumption \ref{asmp:taylor_residual} regarding the Taylor residuals is the most important, which expresses the degree of nonlinearity of models.
When the number of parameters increases, this assumption is satisfied if the derivative tensor created by the third derivatives has smaller eigenvalues.
In other words, if the change in the tensor has a volume in a different direction from the existing eigenvectors, it becomes easier to satisfy this assumption.
A simple scenario in which this is achieved is the case where interrelationships between parameters are sparse, such that the Fisher information matrix is block-diagonal.

A typical approach of the above block-diagonal differential tensor is the examples given in Section \ref{sec:application}, such as parallel deep neural networks and ensemble learning.
Simply put, dividing a large number of parameters into small groups is an effective way to reduce asymptotic risks, as the parameters influence each other only within each group.
A proper partitioning on models can control the volume of differential tensors and achieve a small asymptotic risk, whereas densely connected neural networks cannot satisfy this situation.
This may be useful as a guideline for designing the architecture of large models in the future.

\subsection{Validity of Assumption \ref{asmp:cross_variance}} \label{sec:numerical}
We provide a numerical validation of Assumption \ref{asmp:cross_variance}, which is difficult to verify theoretically
We generate synthetic data from the following linear and nonlinear models:
\begin{align*}
    &\text{(M1)}\ \  y_i = x_i^\top \theta^* + \epsilon_i, \mbox{~~and~~}
    \text{(M2)}\ \  y_i = \frac{1}{p}\sum_{j=1}^p x_{ij} e^{\theta_j^* x_{ij}} + \epsilon_i,
\end{align*}
where $x_{ij}$ independently follows a truncated normal distribution $\mathcal{TN}_{1}(0, 1)$ and $\epsilon_i$ also follows truncated normal distribution $\mathcal{TN}_1(0, 1)$. 
For ratios $\gamma' = p/n  \in \{ 2, 5, 10\}$, we set $p \in \{10,20,...,500\}$ and $n = p / \gamma'$.
For $\tau$, we set $\tau = \overline{\tau} = 0.01$.
Subsequently, we generate the Jacobi matrix $\hat J_{n,p}$ $10$ times and calculate the following off-diagonal term that appears in Assumption \ref{asmp:cross_variance}:
\begin{align}
    \frac{1}{n^2} \sum_{i < j} J_i^\top S_{n,p} J_j.\label{eq:cross1}
\end{align}

We plot the mean and standard deviation of \eqref{eq:cross1} from the replication in Figure \ref{fig:cross0}.
The horizontal axis is the number of parameters, $p$, and the vertical axis shows the value of \eqref{eq:cross1}.
We can observe that the term \eqref{eq:cross1} concentrates around $0$ as $p$ grows.
For the nonlinear model (M2), the speed of the convergence to $0$ is faster.

\begin{figure}[htbp]
    \begin{subfigure}[t]{.48\textwidth}
        \centering
        \includegraphics[width=0.97\hsize]{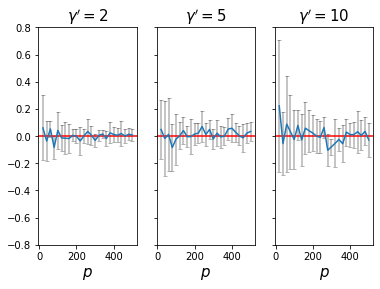}
        \caption{Linear model (M1)}
    \end{subfigure}
    \begin{subfigure}[t]{.48\textwidth}
        \centering
        \includegraphics[width=\hsize]{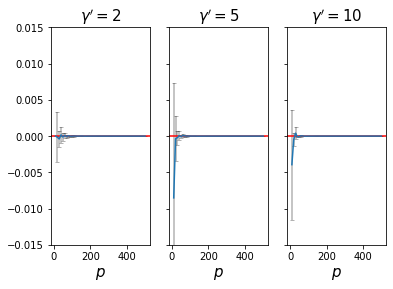}
        \caption{Nonlinear model (M2)}
    \end{subfigure}
    \caption{Value of the term \eqref{eq:cross1} with respect to $p$. The mean (blue curve) and standard deviation (grey bar) from the $10$ replications are plotted. 
    }
    \label{fig:cross0}
\end{figure}

\section{Conclusion and Future Research}

We analyze the asymptotic risk of the penalized maximum likelihood estimator in the large model limit.
Our theorems describe errors using a wide range of models, including deep neural networks, whereas previous studies have only analyzed linear-in-feature models. 
The derived asymptotic risk bounds can account for both the double descent and the regularized risk, depending on the setting of the penalty term. 
We further establish regularity conditions for likelihood models to follow our theorem and demonstrate that parallel deep neural networks and ensemble learning satisfy those conditions.

Our results may contribute to analysis on modern large-scale models according to the regular conditions of asymptotic risk analysis.
Until now, asymptotic risk analysis has only been able to analyze simple linear-in-feature models, and thus could only provide abstract implications for actual complex models and methods such as deep learning.
However, our study establish a way to analyze such complex models, hence we can rigorously analyze asymptotic risks of various modern statistics and machine learning methods.
For example, it may be possible to explore and design an architecture of deep neural networks that are suitable for the double descent phenomenon.

One fruitful future direction of this study is handling the singularity of models, i.e. when the limit of the Fisher information matrix has zero eigenvalues.
In large-scale models, the Fisher information matrices are sometimes singular, thus it is not easy to apply our theory in a straightforward way.
To solve this problem, we need to introduce a specific theory that can handle the singularity of models.
It is challenging and an important future research direction.


\appendix

\section{Calculation on the Extended Stieltjes transform} \label{sec:stieltjes}
We provide a simple calculation on the extended Stieltjes transform term $\lim_{a \to 0} h_{\gamma,\overline{\tau}}(a) $, and provide a proof of Proposition \ref{prop:double_descent}, Example \ref{example:point_spectrum}, and some calculation for Figure \ref{fig:reguralized}.
Although we consider the limit as $p$ and $n$ grows, we omit $p, n \to \infty$ and just write $\lim_{p/n \to \gamma}$ for simplicity.

\subsection{\texorpdfstring{$\overline{\tau}=0$}{LG} case}
We provide the following proof.
\begin{proof}[Proof of Proposition \ref{prop:double_descent}]
Since we can set $\tau = 0$, we have $h^{(0)}(a) = \frac{1}{a}$.
Using this form, we obtain a formula of $h(a) = h_{\gamma,\overline{\tau}}(a)$ as
\begin{align}
    h(a) = - \frac{1}{a - \frac{1}{\gamma(1 + h(a))}}. \label{eq:ha}
\end{align}

For $\gamma \in (0,1)$ case, this form \eqref{eq:ha} and some calculation yields
    $-a h(-a)^2 + (1-1/\gamma + 1) h(-1) = 0$.
We set $a \searrow 0$ and obtain $h(0) = \gamma/(1-\gamma)$.

For $\gamma \in (1,\infty)$ case, we solve the quadratic equation \eqref{eq:ha} and obtain
\begin{align*}
    h(a) = \frac{-(a-1/\gamma + 1) + \sqrt{(a - 1/\gamma + 1)^2 - 4a}}{2a}.
\end{align*}
By the l'Holital's rule, we obtain
\begin{align*}
    \lim_{a \to 0} h(a) = \lim_{a \to 0} \frac{-1 + \frac{2(a - 1/\gamma + 1) - 4}{2 \sqrt{(a - 1/\gamma + 1)^2 - 4a}}}{2} = \frac{1}{\gamma - 1}.
\end{align*}
Then, we obtain the statement.
\end{proof}

\subsection{\texorpdfstring{$\overline{\tau}>0$}{LG} case}
We consider the parameter setting and derive several formulation with various spectral measures $\xi$ by $\lim_{p \to \infty} F_p^*$.

\textbf{Dirac measure case}:
We firstly consider that $\xi$ is a Dirac measure at $\check{\lambda}$ and derive the formulation in Example \ref{example:point_spectrum}.
By the setting, we obtain
\begin{align*}
    h_{\gamma, \overline{\tau}}^{(0)}(a) = \frac{1}{ \overline{\tau} \check{\lambda}/\gamma -a}.
\end{align*}
Using this form, we obtain 
\begin{align*}
    h(a) = h_{\gamma, \overline{\tau}}(a) = \frac{1}{{\overline{\tau}\check{\lambda}}/{\gamma} - (a - 1/(\gamma(1+h(a)))}.
\end{align*}
We solve this equation and achieve
\begin{align*}
    h(a) = \frac{-(\overline{\tau}\check{\lambda}/\gamma - a + 1/\gamma - 1) + \sqrt{(\overline{\tau}\check{\lambda}/\gamma - a + 1/\gamma - 1)^2 + 4(\overline{\tau}\check{\lambda}/\gamma - a)}}{2 (\overline{\tau}\check{\lambda}/\gamma - a)},
\end{align*}
and the l'Hopital's rule generates the statement.

\textbf{Uniform measure case}: 
We consider that $\xi$ is a uniform measure on $[\underline{c}, \overline{c}]$ with $\underline{c}>\overline{c} > 0$.
Then, a simple calculation gives
\begin{align*}
     h_{\gamma, \overline{\tau}}^{(0)}(a) = \frac{\gamma}{\overline{\tau} (\overline{c} - \underline
     c)} \log \left(1 + \frac{\overline{\tau} (\overline{c} - \underline
     c)/\gamma}{\overline{\tau}\underline{c}/\gamma + a} \right).
\end{align*}
Using the form of $h(a) = h_{\gamma,\overline{\tau}}(a)$, we obtain the following equation with $a=0$:
\begin{align*}
    \exp \left( \frac{\tau (\overline{c} - \underline
     c)}{\gamma}h(0) \right) = 1 + \frac{\overline{\tau}(\overline{c} - \underline
     c)}{\overline{\tau} \underline{c} + (1+h(0))^{-1}}.
\end{align*}
We can find a root of this equation and achieve the middle panel in Figure \ref{fig:reguralized}.

   

\textbf{Semi-circular measure case}:
We consider that $\xi$ is a semi-circular measure with its center $c > 1$ and radius $1$, that is, $\xi(A) = \frac{2}{\pi} \int_A \sqrt{1-(x-c)^2} 1_{[c-1,c+1]}(x) dx$ for a Borel subset $A \subset \R$.
With this setting, we prepare a shifted semi-circular measure $\xi'(A) = \frac{2}{\pi} \int_A \sqrt{1-x^2} 1_{[-1,1]}(x) dx$ and obtain
\begin{align*}
    h^{(0)}_{\gamma, \overline{\tau}} (a) = \int_\R \frac{1}{\overline{\tau}t / \gamma + \overline{\tau}c/\gamma - a} d \xi'(t) = 2 \overline{\tau}c/\gamma - 2a - 2\sqrt{(\overline{\tau}c / \gamma - a)^2 - 1}.
\end{align*}
The last inequality follows the Stieltjes transform for $\xi'$ (e.g., see \cite{borot2017introduction}).
We substitute it into the definition of $h(a) =h_{\gamma,\overline{\tau}}(a) $ with $a=0$, and obtain the following equation:
\begin{align*}
    h(0)^3\frac{1}{4} + h(0)^2 \frac{\gamma - \overline{\tau} c}{4\gamma} + h(0)\frac{1 + \gamma - \overline{\tau}c}{\gamma} + 1 = 0.
\end{align*}
We find its root $h(0)$ and plot it in the right panel in Figure \ref{fig:reguralized}.

\section{Proof of Corollary \ref{cor:reg}}
\begin{proof}[Proof of Corollary \ref{cor:reg}]
We calculate the Taylor expansion of $g_\theta$ around $\theta^*$ as
\begin{align*}
    g_{\hat\theta}(x) - g_{\theta^*}(x) = \partial_\theta^\top g_{\theta^*}(x) (\hat\theta - \theta^*) + (\hat\theta - \theta^*)^\top \partial_\theta \partial_\theta^\top g_{\check\theta}(x)(\hat\theta - \theta^*),
\end{align*}
where $\check\theta_j$ lies in the interval between $\theta^*_j$ and $\hat\theta_j$.
Owing to the nonlinear model with truncated normal noise, we see $\Ep[\partial_\theta g_{\theta^*}(x) \partial_\theta^\top g_{\theta^*}(x)] = \Ep[\partial_\theta \log f_{\theta^*}(x) \partial_\theta^\top \log f_{\theta^*}(x) / \epsilon^2] = F_p^* / \Ep[\epsilon^2]$ holds.
Therefore, we obtain
\begin{align*}
    \Ep\qty[|{\partial_\theta^\top g_{\theta^*}(x)(\hat\theta - \theta^*)}|^2 \mid (x_1, y_1), \dots, (x_n, y_n)] = \frac{1}{\Ep[\epsilon^2]}(\hat\theta - \theta^*)^\top F_p^* (\hat\theta - \theta^*).
\end{align*}
By triangle inequality,
\begin{align*}
    \|{g_{\hat\theta} - g_{\theta^*}}\|_{L^2}^2 &\leq \frac{2}{\Ep[\epsilon^2]}\|\hat{\theta} - \theta^*\|_{F^*_p}^2 + 32r^4 \sup_x \|{\partial_{\theta} \partial_{\theta}^\top g_{\theta^*} (x)}\|_{op}^2.
\end{align*}
Combined with Theorem \ref{thm:main} and assumption $\sup_x \|{\partial_{\theta} \partial_{\theta}^\top g_{\theta^*} (x)}\|_{op}^2 = O_\Pp(1)$, we obtain the desired result.
\end{proof}

\section{Proof of Lemma \ref{lem:independence}}

\begin{proof}[Proof of Lemma \ref{lem:independence}]
    We prove the lemma by showing the convergence in moments. 
    Under Assumption \ref{asmp:cross_variance}, the second moment of $(1/n^2) \sum_{i,j=1,i\neq j}^n J_i^\top S_{n,p} J_j$ is calculated as follows:
    \begin{align}
        &\frac{1}{n^4}\sum_{i,j=1,i\neq j}^n \sum_{u,v=1,u\neq v}^n \Ep[J_i^\top S_{n,p} J_j J_u^\top S_{n,p} J_v]\notag  \\
        &= \frac{1}{n^4}\sum_{i,j=1,i\neq j}^n \sum_{u,v=1,u\neq v}^n \Ep[\Ep[J_i^\top S_{n,p} J_j J_u^\top S_{n,p} J_v | J_i, J_j, J_u, J_v] ]\notag  \\
        &= \frac{1}{n^4}\sum_{i,j=1,i\neq j}^n \sum_{u,v=1,u\neq v}^n \Ep[J_i^\top L_{n,p} J_j J_u^\top L_{n,p} J_v + U_{p, i, j} ]\notag  \\
        &= \frac{1}{n^4} \sum_{i,j=1,i\neq j}^n \tr(\Ep[J_i J_i^\top L_{n,p} J_j J_j^\top L_{n,p}]) + o(1).
    \end{align}
    Because $\Ep[J_i J_i^\top] = F_p^*$, we obtain
    \begin{align*}
        \Var\qty(\frac{1}{n^2} \sum_{i,j=1,i\neq j}^n J_i^\top S_{n,p} J_j) &\leq \frac{n(n-1)}{n^4} \tr\qty(L_{n,p}^2 {F_p^*}^2) + o(1)\\
        &\leq \frac{n(n-1)}{n^4} \overline{\lambda}^2 \tr\qty(L_{n,p}^2) + o(1)\\
        &= o(1).
    \end{align*}
    The last equality holds since $\tr(L_{n,p}^2) = o(p^2)$ by assumption.
    The conclusion follows from the fact that convergence in moment implies convergence in probability.
\end{proof}

\section{Proof of Main Results}

We define a normalized vector $\Tilde{J}_i :=(1/\sqrt{p}) {F^*_p}^{-1/2} \hat J_i $, and a $p \times n$ matrix $\Tilde{J}_{n,p} := (\tilde{J}_1, \dots, \tilde{J}_n) = (1/\sqrt{p}) {F^*_p}^{-1/2} \hat J_{n, p}$. 
Note that $\tilde{J}_i$ is an independent and identically distributed random vector with mean $\mathbb{E}[\tilde{J}_i] = O_{p \times n}$ and covariance $\Cov(\tilde{J}_i) = I_p / p$.

We start the proof from the basis decomposition \eqref{eq:mle_diff}.
We restate it as
\begin{align*}
     \theta^* - \hat\theta = V_0(V_1 + V_2 + B_0), 
\end{align*}
where the terms are recalled as
\begin{align*}
    &V_0 = (\partial_\theta \partial_\theta^\top M_n(\theta^*) + \tau I_p)^{-1} (\hat{F}_{n,p} + \tau I_p),\\
    &V_1 = (\hat{F}_{n,p} + \tau I_p)^{-1}\partial_\theta M_n(\theta^*), \\
    &V_2 = (\hat{F}_{n,p} + \tau I_p)^{-1}R,\\
    &B_0 = (\hat{F}_{n,p} + \tau I_p)^{-1}\tau \theta^*.
\end{align*}

\subsection{Bound Individual Terms}

As the first step, we provide the following lemma to analyze the effect of $V_0$.
\begin{lemma}\label{lem:bound_v0}
    Suppose Assumption \ref{asmp:fisher_residual} holds.
    Then, for any $\tau > 0$, the following inequality holds as $n,p \to \infty$:
    \begin{align*}
        \left\|F_p^{* 1/2} V_0 F_p^{* -1/2}\right\|_{op} \leq 1 + \sqrt{\lambda^*}.
    \end{align*}
\end{lemma}
\begin{proof}[Proof of Lemma \ref{lem:bound_v0}]
As preparation, we recall the definition of the Fisher residual $\hat W_{n, p}$ as
\begin{align*}
    \hat W_{n, p} & = \frac{1}{n}\sum_{i=1}^m w(z_i)\\
    &= \frac{1}{n} \sum {\frac{1}{f_{\theta^*}(z_i)^2} \qty(\partial_\theta f_{\theta^*}(z_i)) \qty(\partial_\theta f_{\theta^*}(z_i))^\top} - \frac{1}{n}\sum {\frac{1}{f_{\theta^*}(z_i)} \partial_\theta \partial_\theta^\top f_{\theta^*}(z_i)} \\
    &=\hat F_{n, p}  -  \partial_\theta \partial_\theta^\top M_n(\theta^*) .
\end{align*}

Recall $V_0 = ( \hat F_{n, p} - \hat W_{n, p} + \tau I_p)^{-1} (\hat{F}_{n,p} + \tau I_p)$. We bound the following norm as
\begin{align*}
    \left\|F_p^{* 1/2} V_0  F_p^{* -1/2}\right\|_{op} &= \norm{{F^*_p}^{1/2} \qty( \hat F_{n, p} - \hat W_{n, p} + \tau I_p)^{-1} \qty( \hat{F}_{n,p} + \tau I_p) F_p^{* -1/2}}_{op}\\
    &\leq 1 + \underbrace{\norm{{F^*_p}^{1/2} \qty(\hat F_{n, p} - \hat W_{n, p} + \tau I_p)^{-1} \hat W_{n, p} {F^*_p}^{-1/2}}_{op}}_{=: V_{0,1}}.
\end{align*}
Here, we utilize an inequality for square matrices $A$ and $B$ as $\|{A^{-1} B}\|_{op} \leq 1 + \|{A^{-1} (B - A)}\|_{op}$, with substituting $A \leftarrow ( \hat F_{n, p} - \hat W_{n, p} + \tau I_p) {F^*_p}^{-1/2}$ and $B \leftarrow ( \hat F_{n, p} + \tau I_p){F^*_p}^{-1/2}$.

To bound the term $V_{0,1}$, we have to show that $\hat F_{n, p} - \hat W_{n, p} + \tau I_p$ is positive definite matrix with some $\tau$ with probability approaching to $1$.
To this end, we employ the following matrix version of the Bernstein inequality and bound an operator norm of $\hat W_{n, p}$.
The following lemma is a slight extension of the results from \cite{tropp2015introduction}.

\begin{lemma}\label{lem:matrixbernstein}
    Let $X_i$ be a sequence of independent and identically distributed random Hermitian matrices with dimension $p$. Assume $\Ep[X_i] = 0$. Define $v^2 := \norm{\Ep[X_1^2]}_{op}$. Suppose 
    \begin{align*}
        \norm{\Ep[X_i^k]}_{op} \leq \frac{1}{2} k! v^2 {C_L}^{k-2}
    \end{align*}
    for some $C_L > 0$.
    Then,
    \begin{align*}
        \Pp\qty(\norm{\frac{1}{n} \sum_{i=1}^n X_i}_{op} \geq t) \leq 2 p \exp\qty(\frac{-nt^2/2}{v^2 + C_L t /3})
    \end{align*}
    holds for all $t \geq 0$.
\end{lemma}

We apply Lemma \ref{lem:matrixbernstein}  to $\hat W_{n, p} =n^{-1} \sum_{i=1}^n w(z_i)$ and obtain the following for any $t > 0$:
\begin{align*}
    P\qty(\norm{\hat W_{n, p}}_{op} \geq t) \leq 2p \exp\qty( \frac{-n t^2 / 2}{\nu_p^2 + \kappa_p t / 3} ) =: p_W(t).
\end{align*}
We substitute $t \leftarrow \tau / 2$ so that $ - \hat W_{n, p} + (\tau/2) I_p$ is positive semi-definite with probability at least 
\begin{equation}
     1-p_W(\tau/2) = 1 - 2p \exp\qty( \frac{-n \tau^2 / 8}{\nu_p^2 + \kappa_p \tau / 6} ).\label{eq:prob1}
\end{equation}

On the event with probability $1-p_W(\tau/2)$, we obtain that $\hat F_{n, p} - \hat W_{n,p} + \tau I_p$ is positive definite because $\hat F_{n,p}$ is defined as the sum of positive semi-definite matrices.
Recall that $\hat F_{n, p} = (1/n) \sum_{i=1}^n J_i J_i^\top$ and $J_i J_i^\top$ is positive semi-definite by the quadratic form, then $\hat F_{n, p}$ is always a positive semi-definite matrix.

 
We are now ready to bound (T1) by using the positive definiteness.
We utilize the decomposition of an operator norm and obtain the following with probability at least  $1-p_W(\tau/2)$:
\begin{align*}
    V_{0,1}&\leq \norm{{F^*_p}^{1/2}}_{op} \norm{(\hat F_{n, p} - \hat W_{n, p} + \tau I_p)^{-1}}_{op}\norm{\hat W_{n, p}}_{op} \norm{{F^*_p}^{-1/2}}_{op}\\
    &\leq \sqrt{\frac{\overline{\lambda}}{\underline{\lambda}}} \qty(\lambda_{\min}(\hat F_{n, p} + (\tau/2) I_p - \hat W_{n, p} + (\tau/2) I_p))^{-1} \tau / 2\\
    &\leq \sqrt{\lambda^*}
\end{align*}
where the last inequality follows from $\lambda_{\min}(\hat F_{n, p} + (\tau/2) I_p - \hat W_{n, p} + (\tau/2) I_p) \geq \tau / 2$.

Finally, we derive the limit of the probability $1-p_W(\tau/2)$ as $n,p \to \infty$.
From Assumption \ref{asmp:fisher_residual} and $\tau^2 \succ s_p$, we have 
\begin{align*}
    \frac{n\tau^2}{\nu_p^2 \log p} \to \infty \ \ \text{and} \ \  \frac{n\tau}{\kappa_p \log p} \to \infty,
\end{align*}
hence $p_W(\tau/2) = o(1/p)$. 
Therefore, the probability \eqref{eq:prob1} is $1$ at the limit.
\end{proof}

\begin{lemma}\label{lem:bound_v1}
    Suppose Assumption \ref{asmp:basic} and \ref{asmp:cross_variance} holds.
    Then, there exists a variable $V_{1,3}$ such as
    \begin{align*}
        \|V_1\|_{F_p^*}^2 \leq V_{1,3} \to \lambda^* \lim_{a \to 0} h_{\gamma, \Bar{\tau}}(a),
    \end{align*}
    in probability, as $n,p \to \infty$.
\end{lemma}
\begin{proof}[Proof of Lemma \ref{lem:bound_v1}]
We recall the decomposition in Section \ref{sec:proof_outline} as
\begin{align*}
    \|V_1\|_{F_p^*}^2 
     &=\tr\qty(( \hat F_{n, p} + \tau I_p)^{-1}{F^*_p} ( \hat F_{n, p} + \tau I_p)^{-1} n^{-2} \hat J_{n, p} \hat J_{n, p}^\top   ) \\
     & \quad +\tr\qty(( \hat F_{n, p} + \tau I_p)^{-1}{F^*_p} ( \hat F_{n, p} + \tau I_p)^{-1}(\partial_\theta M_n(\theta^*) \partial_\theta M_n(\theta^*)^\top - n^{-2} \hat J_{n, p} \hat J_{n, p}^\top)  ) \\
     &=: V_{1,1} + V_{1,2}.
\end{align*}
We bound each of them.

Our primary interest is to derive a bound for $V_{1,1}$ at the limit with $n,p \to \infty$ and $p/n \to \gamma$.
Recall $\Tilde{J}_{n,p}$ is defined as $\Tilde{J}_{n,p} = (\Tilde{J}_1, \dots, \Tilde{J}_n) = (1/\sqrt{p}) {F^*_p}^{-1/2} \hat J_{n, p}$.
Then, $\Tilde{J}_i$ is isotropic. i.e. $\mathbb{E}[\Tilde{J}_i] = 0$, $\Cov(\Tilde{J}_i) = I_p / p$.
We obtain ${F^*_p}^{-1/2} \hat J_{n, p} \hat J_{n, p}^\top {F^*_p}^{-1/2} = p \Tilde{J}_{n,p} \Tilde{J}_{n,p}^\top$.
To handle the limit, we apply the extended version of the Marchenko-Pastur law in Lemma \ref{lem:marchenko}.
Let $\mu_{n,p}$ be a spectral measure of $\Tilde{J}_{n,p} \Tilde{J}_{n,p}^\top + (n/p) \tau {F^*_p}^{-1}$, and
$\mu$ be a spectral measure of its limit as in Lemma \ref{lem:marchenko} with substitution $T_{n,p} \leftarrow \Tilde{J}_{n,p}$ and $E_p \leftarrow (n/p)\tau {F^*_p}^{-1}$.

When $\gamma < 1$, we assume $p < n$ holds without loss of generality.
We bound the term $V_{1,1}$ as
\begin{align}
    V_{1,1}&=\frac{1}{n}\tr\qty(\qty( \hat F_{n, p} + \tau I_p)^{-1} {F^*_p} \qty( \hat F_{n, p} + \tau I_p)^{-1} \frac{1}{n} \hat J_{n, p} \hat J_{n, p}^\top ) \notag \\
    &= \frac{1}{n}\tr\qty({F^*_p}^{1/2}\qty( \hat F_{n, p} + \tau I_p)^{-1} {F^*_p}^{1/2} {F^*_p}^{1/2} \qty( \hat F_{n, p} + \tau I_p)^{-1} {F^*_p}^{1/2} {F^*_p}^{-1/2}\frac{1}{n} \hat J_{n, p} \hat J_{n, p}^\top {F^*_p}^{-1/2}) \notag \\
    &= \frac{1}{p}\tr\qty(\qty(\Tilde{J}_{n,p} \Tilde{J}_{n,p}^\top + (n/p) \tau {F^*_p}^{-1})^{-2} \Tilde{J}_{n,p} \Tilde{J}_{n,p}^\top) \notag \\
    &= \frac{1}{p}\tr\qty(\qty(\Tilde{J}_{n,p} \Tilde{J}_{n,p}^\top + (n/p) \tau {F^*_p}^{-1})^{-1}) - \frac{1}{p}\tr\qty(\qty(\Tilde{J}_{n,p} \Tilde{J}_{n,p}^\top + (n/p) \tau {F^*_p}^{-1})^{-2} (n/p) \tau {F^*_p}^{-1}) \notag \\
    &\leq \frac{1}{p}\tr\qty(\qty(\Tilde{J}_{n,p} \Tilde{J}_{n,p}^\top + (n/p) \tau {F^*_p}^{-1})^{-1}) \notag \\
    &= \int \frac{1}{\lambda}\1_{\lambda \geq \lambda_+} \dd{\mu_{n,p}(\lambda)}, \notag
\end{align}
where the last equality follows from Assumption \ref{asmp:basic}.

When $\gamma > 1$, we assume $p > n$ holds.
For $\overline{\tau} > 0$, we can further assume $\tau > \overline{\tau}/2$ and $p/n < 2\gamma$.
By a similar argument as above
\begin{align}
    V_{1,1} &\leq \int \frac{1}{\lambda}\1_{\lambda \geq \overline{\tau}/ (4 \gamma \overline{\lambda}) } \dd{\mu_{n,p}(\lambda)}. \notag
\end{align}
For $\overline{\tau} = 0$, we bound $V_{1,1}$ as
\begin{align}
    V_{1,1} &= \frac{1}{p}\tr\qty(\qty(\Tilde{J}_{n,p} \Tilde{J}_{n,p}^\top + (n/p) \tau {F^*_p}^{-1})^{-2} \Tilde{J}_{n,p} \Tilde{J}_{n,p}^\top) \notag \\
    &\leq \frac{\overline{\lambda}}{p}\tr\qty({F_p^*}^{-1}\qty(\Tilde{J}_{n,p} \Tilde{J}_{n,p}^\top + (n/p) \tau {F_p^*}^{-1})^{-2} \Tilde{J}_{n,p} \Tilde{J}_{n,p}^\top) \notag \\
    &= \frac{\overline{\lambda}}{p}\tr\qty(\qty(\Tilde{J}_{n,p} \Tilde{J}_{n,p}^\top {F_p^*} + (n/p) \tau I_p)^{-1} ) - \frac{\overline{\lambda}}{p}\tr\qty(\qty(\Tilde{J}_{n,p} \Tilde{J}_{n,p}^\top {F_p^*} + (n/p) \tau I_p)^{-2}  (n/p) \tau I_p). \label{eq:key0}
\end{align}
The inequality follows from $\tr(AB) \leq \lambda_{\max}(A) \tr(B)$ for positive semi-definite matrices $A$ and $B$. 
Since $\Tilde{J}_{n,p} \Tilde{J}_{n,p}^\top {F_p^*}$ is rank $n$, we have
\begin{align*}
    \tr\qty(\qty(\Tilde{J}_{n,p} \Tilde{J}_{n,p}^\top {F_p^*} + (n/p) \tau I_p)^{-2}  (n/p) \tau  I_p)
    \geq (p - n) \qty(\frac{n}{p} \tau)^{-1}.
\end{align*}
Using this inequality and Condition (v) in Assumption \ref{asmp:basic}, we can evaluate \eqref{eq:key} as
\begin{align}
    \eqref{eq:key0} &\leq \frac{\overline{\lambda}}{p} \sum_{i=1}^p \frac{1}{\lambda_i(\Tilde{J}_{n,p} \Tilde{J}_{n,p}^\top {F^*_p} + (n/p)\tau I_p)} - \overline{\lambda} \frac{p - n}{p} \qty(\frac{n}{p} \tau)^{-1} \notag \\
    &= \frac{\overline{\lambda}}{p} \sum_{i=1}^p \frac{1}{\lambda_i(\Tilde{J}_{n,p} \Tilde{J}_{n,p}^\top {F^*_p} + (n/p)\tau I_p)} \1_{\lambda_i(\Tilde{J}_{n,p} \Tilde{J}_{n,p}^\top F_p^*) > 0} \notag\\
    &= \frac{\overline{\lambda}}{p} \sum_{i=1}^p \frac{1}{\lambda_i(\Tilde{J}_{n,p} \Tilde{J}_{n,p}^\top {F^*_p} + (n/p)\tau I_p)} \1_{\lambda_i(\Tilde{J}_{n,p} \Tilde{J}_{n,p}^\top) \geq \lambda_+} \notag\\
    &\leq \frac{\overline{\lambda}}{\underline{\lambda}} \frac{1}{p} \sum_{i=1}^p \frac{1}{\lambda_i(\Tilde{J}_{n,p} \Tilde{J}_{n,p}^\top + (n/p)\tau {F^*_p}^{-1} )} \1_{\lambda_i(\Tilde{J}_{n,p} \Tilde{J}_{n,p}^\top) \geq \lambda_+} \notag\\
    &= \lambda^* \int \frac{1}{\lambda}\1_{\lambda \geq \lambda_+} \dd{\mu_{n,p}(\lambda)}. \notag
\end{align}
In both cases, to bound $V_{1,1}$, it is important to bound an integral with the following form:
\begin{align}
    \int \frac{1}{\lambda}\1_{\lambda \geq c} \dd{\mu_{n,p}(\lambda)} \label{eq:key}
\end{align}
with a constant $c \in \{\lambda_+,\overline{\tau}/ (4 \gamma \overline{\lambda}) \} $ such as $c > 0$.
Because $\lambda \mapsto (1/\lambda)\1_{\lambda \geq c}$ is a bounded function, we can use Lemma \ref{lem:marchenko} to obtain
\begin{align}
    \lim_{p,n \to \infty, p/n \to \gamma} \int \frac{1}{\lambda}\1_{\lambda \geq c} \dd{\mu_{n,p}(\lambda)} = \int \frac{1}{\lambda}\1_{\lambda \geq c} \dd{\mu(\lambda)} \leq \int \frac{1}{\lambda} \dd{\mu(\lambda)}. \label{eq:limit_stil}
\end{align}
In summary, for the case with $\gamma > 1$ and $\overline{\tau} = 0$, we have $\limsup V_{1,1} \leq \lambda^*\int (1/\lambda) \dd{\mu(\lambda)}$.
For the other cases, $\limsup V_{1,1} \leq \int (1/\lambda) \dd{\mu(\lambda)}$ holds.
The bound in \eqref{eq:limit_stil} corresponds to the limit of the Stieltjes transform of $\mu$, that is, we obtain
\begin{align*}
    \int \frac{1}{\lambda} \dd{\mu(\lambda)} &= \lim_{a \to 0} h_{\gamma, \Bar{\tau}}(a).
\end{align*}

For $V_{1,2}$, we simply rewrite the term as
\begin{align*}
    V_{1,2} 
    &= \frac{1}{n^2} \sum_{i,j=1,i\neq j}^n J_i^\top \qty( \hat F_{n, p} + \tau {I_p}^{-1})^{-1} F_p^*\qty( \hat F_{n, p} + \tau {I_p}^{-1})^{-1} J_j.
\end{align*}
By Assumption \ref{asmp:cross_variance}, we obtain $V_{1,2} \to 0$ as $n,p \to \infty$.

To obtain the statement, we define the existing variable $V_{1,3} := n^{-1}\mathrm{tr} ({F^*_p} ( \hat F_{n, p} + \tau I_p)^{-1} ) + V_{1, 2}$ and utilize $\lambda^* \geq 1$.
\end{proof}

\begin{lemma}\label{lem:bound_R}
    Suppose Assumption \ref{asmp:taylor_residual} holds.
    Then, with probability at least $1 - 2/p$, there exists some constant $C$ such that 
    \begin{align*}
        \norm{R}^2_2 \leq C\qty( r^6 \sum_{j=1}^p \qty(\alpha_p^j)^2 \vee r^4 \sum_{j=1}^p \qty(\beta_p^j)^2 ),
    \end{align*}
    as $p \to \infty$.
\end{lemma}

\begin{proof}
We prepare some notation.
Let $\mN := \mN(\Theta, \norm{\cdot}, \delta)$ be a $\delta$-covering number of $\Theta$, that is, a minimum number of balls with radius $\delta$ to cover $\Theta$ in terms of $\|\cdot\|$.
We further define a quadratic form
\begin{align*}
    \ell_{\bar \theta, \underline{\theta}}^j(z) = (\bar\theta - \theta^*)^\top U^j_p(\underline{\theta},z) (\bar\theta - \theta^*),
\end{align*}
for parameters $\Bar{\theta}, \underline{\theta} \in \Theta$.
Note that it satisfies $R_j = n^{-1} \sum_{i=1}^n \ell_{\hat \theta, \check{\theta}^j}^j(z_i)$ for any $j = 1,...,p$.
We also define its family 
\begin{align*}
    \mG^j &= \qty{\ell_{\bar \theta, \underline{\theta}}^j(z) \mid \bar\theta, \underline{\theta} \in \Theta}.
\end{align*}
We introduce the empirical process type notations $\Pp_n$ and $\Pp$ such that $\Pp_n g = (1/n) \sum_{i=1}^n g(z_i)$ and $\Pp g = \Ep_z[g(z)]$ for a function $g$.
Also, we define a sup norm on a set as $\norm{\Pp_n}_{\mG} := \sup_{g \in \mG} \abs{\Pp_n g}$.

We can bound the Taylor residual $R_j$ as
\begin{align*}
    R_j = \Pp_n \ell_{\hat \theta, \check{\theta}^j}^j \leq \sup_{\Bar{\theta}, \underline{\theta} \in \Theta} \Pp_n \ell_{\bar \theta, \underline{\theta}}^j = \|\Pp_n\|_{\mG^j},
\end{align*}
and we will develop an upper bound of the uniform bound for $\|\Pp_n\|_{\mG^j}$.
This proof utilizes a technique with the Rademacher complexities, which has two steps: (i) develop a probability bound for a deviation from its expectation, and (ii) bound the expectation of the Rademacher complexity.

For the first step, we apply the Talagrand's inequality (Theorem 3.4.3 in \cite{gine2016mathematical}) with an expectation of supremum replaced by the Rademacher process.
We define $\sigma^2 = \sup_{\ell \in \mG^j} \Pp \ell^2$, and  $\mR_n$ as the Rademacher complexity of $\mG^j$  such as
\begin{equation*}
    \mR_n := \sup_{\ell \in \mG^j} \abs{ \frac{1}{n} \sum_{i} u_i \ell(z_i) }
\end{equation*}
for Rademacher variables $u_i$ wuch as $\Pp(u_i = 1) = \Pp(u_i = -1) = 1/2$, which is independent of $z_i$.
Then, by the Talagrand's inequality, we obtain the following result for any $t > 0$:
\begin{align}
    \Pr\qty(\norm{\Pp_n - \Pp}_{\mG^j} \geq 3 \mathbb{E}\qty[\|{\mR_n}\|_{\mG^j}] + 16r^2 \beta_p^j \qty( \qty(\frac{t}{n} + 2\sigma^2) + \frac{35t}{3n})) \leq 2e^{-t}. \label{ineq:talagrand_apply}
\end{align}

As the second step, we bound the expectation term $\mathbb{E}\qty[\|{\mR_n}\|_{\mG^j}]$.
We apply the maximal inequality for Rademacher variables from Corollary 2.2.5 from \cite{van1996weak},
\begin{align}
    \mathbb{E}\qty[\|{\mR_n}\|_{\mG^j}] \leq \frac{C_1}{\sqrt{n}} \int_{0}^{D} \sqrt{1 + \log \mN(\delta, \mG^j, \norm{\cdot}_\infty)} \dd{\delta}, \label{ineq:rademacher_apply}
\end{align}
where $D = \sup_{\ell \in \mG^j} \norm{\ell}_\infty$ and $C_1$ is a universal constant. 
Note that $D \leq 4r^2 \beta_p^j$.

To bound the covering number term $\mN(\delta, \mG^j, \norm{\cdot}_\infty)$, we bound the term by a combination by the covering numbers of $\Theta$.
To this aim, we develop the following inequality:
\begin{align*}
    &\abs{(\bar\theta - \theta^*)^\top U^j_p(\underline{\theta},z) (\bar\theta - \theta^*) - (\bar\theta' - \theta^*)^\top U^j_p(\underline{\theta}',z) (\bar\theta' - \theta^*)}\\
    &\leq \abs{(\bar\theta - \theta^*)^\top U^j_p(\underline{\theta},z) (\bar\theta - \theta^*) - (\bar\theta - \theta^*)^\top U^j_p(\underline{\theta}',z) (\bar\theta - \theta^*)}\\
    &\quad+ \abs{(\bar\theta - \theta^*)^\top U^j_p(\underline{\theta}',z) (\bar\theta - \theta^*) - (\bar\theta' - \theta^*)^\top U^j_p(\underline{\theta}',z) (\bar\theta' - \theta^*)}\\
    &\leq \norm{\bar\theta - \theta^*}_\infty^2 \norm{U^j_p(\underline{\theta},z) - U^j_p(\underline{\theta}',z)}_{op} \\
    &\quad+ \qty(\norm{\bar\theta + \bar\theta'}_\infty + 2\norm{\theta^*}_\infty) \norm{U^j_p(\underline{\theta},z)}_{op} \norm{\bar\theta - \bar\theta'}_\infty\\
    &\leq 4r^2 \alpha_p^j \norm{\underline{\theta} - \underline{\theta}'}_\infty + 4r \beta_p^j \norm{\bar\theta - \bar\theta'}_\infty.
\end{align*}
Hence, a combination of coverings sets for $\Theta$ implies a covering set of $\mG^j$, that is, we obtain
\begin{align*}
    \mN(\delta, \mG^j, \norm{\cdot}_\infty) &\leq \mN\qty(\frac{\delta}{8r^2 \alpha_p^j}, \Theta, \norm{\cdot}_\infty) \times \mN\qty(\frac{\delta}{8r \beta_p^j}, \Theta, \norm{\cdot}_\infty)\\
    &\leq \qty(1 + \qty(\frac{16r^3 \alpha_p^j}{\delta}))^p \qty(1 + \qty(\frac{16r^2\beta_p^j}{\delta}))^p.
\end{align*}
The last inequality follows the setting $\Theta \subset B(0, r) \subset [-r, r]^p$.

We substitute this bound to \eqref{ineq:rademacher_apply} and using $\sqrt{\log(1 + x)} \leq 1 + \log(1 + x)$ for any $x > 0$, then obtain
\begin{align*}
    \mathbb{E}\qty[\|{\mR_n}\|_{\mG^j}] 
    &\leq \underbrace{C_1\sqrt{\frac{p}{n}} \int_{0}^{D} \qty(1 + \log \qty(1 + \frac{16r^3\alpha_p^j}{\delta})) \dd{\delta}}_{=: V_{2,1}}+ \underbrace{ C_1\sqrt{\frac{p}{n}} \int_{0}^{D} \qty(1 + \log \qty(1 + \frac{16r^2\beta_p^j}{\delta})) \dd{\delta}}_{=: V_{2,2}}.
\end{align*}
For the integral part in $V_{2,1}$, we can bound it as
\begin{align*}
    \int_{0}^{D} \qty(1 + \log \qty(1 + \frac{16r^3 \alpha_p^j}{\delta})) \dd{\delta}
    &\leq (D + 16r^3 \alpha_p^j) \log\qty(\frac{D}{16r^3 \alpha_p^j} + 1) \leq 5 r^3 \alpha_p^j.
\end{align*}
Similarly, we obtain
\begin{align*}
    \int_{0}^{D} \qty(1 + \log \qty(1 + \frac{16r^2 \beta_p^j}{\delta})) \dd{\delta} \leq 5 r^2 \beta_p^j.
\end{align*}

We combine the results of \eqref{ineq:talagrand_apply} and \eqref{ineq:rademacher_apply}.
Let $t \leftarrow 2\log p$. Note that $\sigma^2 \leq D^2$. By the union bound argument,
\begin{align*}
    R_j &\leq \mathbb{E}[(\hat\theta - \theta^*)^\top \partial_\theta \partial_\theta^\top \partial_{\theta_j} \log_{\check\theta^j} (z) (\hat\theta - \theta^*)]\\
    &\quad+ C_1\sqrt{\frac{p}{n}} \qty( 5 r^3 \alpha_p^j + 5 r^2 \beta_p^j)\\
    &\quad+ 32r^2 \beta_p^j \qty( \qty(\frac{\log p}{n} + 2\sigma^2) + \frac{35\log p}{3n})\\
    &= O\qty( r^3 \alpha_p^j \vee r^2 \beta_p^j).
\end{align*}
This concludes the proof.

\end{proof}

\subsection{Combine the Results}
We combine all the above lemmas and prove Theorem \ref{thm:variance} and \ref{thm:main}.

\begin{proof}[Proof of Theorem \ref{thm:variance}]
According to the definition of the variance term 
\begin{align*}
    V_{n,p}(\hat{\theta}) = V_0(V_1 + V_2).
\end{align*}
We bound its weighted norm as 
\begin{align*}
    \|V_{n,p}(\hat{\theta})\|_{F_p^*}^2 &= \|F_p^{*1/2} V_0(V_1 + V_2)\|_{2}^2  \\
    & \leq \|F_p^{*1/2} V_0 {F_p^*}^{-1/2} \|_{op}^2 \|F_p^{*1/2} (V_1 + V_2)\|_{2}^2 \\
    & \leq 2\left(1+\sqrt{\frac{\overline{\lambda}}{\underline{\lambda}}}\right)^2 (\|V_1 \|_{F_p^*}^2 + \| V_2\|_{F_p^*}^2).
\end{align*}
The last inequality holds with probability approaching $1$ by Lemma \ref{lem:bound_v0}.
For $\|V_1 \|_{F_p^*}^2$, we apply Lemma \ref{lem:bound_v1} and obtain
\begin{align*}
    \|V_1 \|_{F_p^*}^2 \leq V_{1,3} \to \lambda^* \lim_{a \to 0} h_{\gamma, \Bar{\tau}}(a),
\end{align*}
as $n,p \to \infty$.
For $\| V_2\|_{F_p^*}^2$, we obtain
\begin{align*}
    \| V_2\|_{F_p^*}^2 &\leq  \norm{{F^*_p}^{1/2} \qty( \hat F_{n, p} + \tau I_p)^{-1}}_{op}^2 \norm{ R}^2\\
    &\leq \frac{\lambda_{\max}({F^*_p})}{\tau^2}\norm{R}^2 \\
    &\leq \frac{\overline{\lambda}}{\tau^2} \left\{ r^6 \sum_{j=1}^p \qty(\beta_p^j)^2 \vee r^4 \sum_{j=1}^p \qty(\alpha_p^j)^2 \right\},
\end{align*}
by Lemma \ref{lem:bound_R}. It converges to zero by Assumption \ref{asmp:taylor_residual} and the setting $\tau^2 \succ s_p$.
Combining the results gives the statement.
\end{proof}
\begin{proof}[Proof of Theorem \ref{thm:main}]
We additionally analyze the bias term $B_{n,p}(\hat{\theta}) = V_0 B_0$.
Here, we only have to evaluate the bias term $\|B_0\|_{F_p}^2$ and add it to the result of Theorem \ref{thm:variance}.
We have
\begin{align*}
    \|B_0\|_{F_p}^2 \leq \norm{{F^*_p}^{1/2} \qty( \hat F_{n, p} + \tau I_p)^{-1}}_{op}^2 \norm{\tau\theta^*}^2 \leq \frac{\lambda_{\max}(F_p^*)}{\tau^2} \| \tau \theta^*\|^2 \leq \overline{\lambda}r^2.
\end{align*}
Then, we obtain the statement.
\end{proof}

\subsection{Proof of Matrix Bernstein Inequality}

The following lemma connects from Theorem 6.6.1 in \cite{tropp2015introduction}.
\begin{lemma}[Connection to Lemma \ref{lem:matrixbernstein}] \label{lem:bernstein2}
    Let $Y$ be a random matrix with mean $\mathbb{E}[Y] = O$ and second moment $\mathbb{E}[Y^2] = \Sigma$.
    Then, for any $\theta \geq 0$,
    \begin{align*}
        \Pp(\lambda_{\max}(Y) \geq t) \leq e^{-\theta t} \tr\qty(\exp\qty( \log \mathbb{E} e^{\theta Y} ))
    \end{align*}
\end{lemma}
Note that
\begin{align*}
    \mathbb{E}[e^{\theta Y}] \leq I + \frac{\theta^2}{2} \Sigma + \sum_{k=3}^\infty \theta^k \frac{\mathbb{E}[Y^k]}{k!}.
\end{align*}
For a diagonalizable matrix $A = V \Lambda V^\top \in \R^p$, where $\Lambda = \diag(\lambda_1, \dots, \lambda_p)$ is a diagonal matrix,
$\log A = V (\log \Lambda) V^\top$, where $\log \Lambda$ is a diagonal matrix whose $(i,i)$-th element is $\log \lambda_{i}$. Hence $\norm{\log A} = \log \norm{A}$.

Suppose $\norm{\mathbb{E}[Y^k]}_{op} / \norm{\Sigma}_{op} \leq (k! / 2) L^{k-2}$ for some $L > 0$.
Then,
\begin{align*}
    \norm{\mathbb{E}[e^{\theta Y}]} &\leq 1 + \theta^2 \frac{\norm{\Sigma}}{2} + \sum_{k=3}^\infty \theta^k \frac{\norm{\mathbb{E}[Y^k]}}{k!}= 1 + \theta^2 \frac{\norm{\Sigma}}{2} \qty(\sum_{k=0}^\infty \abs{\theta L}^k).
\end{align*}

Therefore, for $\theta \in (0, 1/L)$,
\begin{align}
    \Pp(\lambda_{\max}(Y) \geq t) &\leq e^{-\theta t} d \exp\qty( \log \lambda_{\max} (E e^{\theta Y}) )\\
    &\leq d e^{-\theta t} \exp\qty( \frac{\theta^2 \norm{\Sigma}}{2(1 - \abs{\theta L})} ).
\end{align}
We obtain the same bound for $\Pp(\lambda_{\max}(-Y) \geq t)$.
Since $\Pp(\norm{Y} \geq t) \leq \Pp(\lambda_{\max}(Y) \geq t) + \Pp(\lambda_{\max}(-Y) \geq t)$, we obtain the desired result by substitution $\theta \leftarrow t/(\norm{\Sigma} + Lt/3)$.

\section{Proof for Applications} \label{sec:appendix_appli}

\subsection{Additive Regression for Parallel Deep Neural Network}
As preparation, we provide a rigorous notation of
supremum norms of the first, second, and third-order derivatives.
For each $m$, and $j,k,\ell \in \Delta_m$, the derivatives are defined as follows:
\begin{align*}
    \omega_{1}^{j} := \sup_{\theta' \in \Theta} \norm{\partial_{\theta_j} g_{\Delta_m(\theta')}}_{L^\infty},~\omega_{2}^{j,k} := \sup_{\theta' \in \Theta} \norm{\partial_{\theta_j} \partial_{\theta_{k}} g_{\Delta_m(\theta')}}_{L^\infty}, \mbox{~and~}\omega_{3}^{j,k,\ell} := \sup_{\theta' \in \Theta}  \norm{\partial_{\theta_j} \partial_{\theta_{k}} \partial_{\theta_{\ell}} g_{\Delta_m(\theta')}}_{L^\infty}.
\end{align*}
We also define its maximum for each $m = 1,...,M$ as
\begin{align*}
    \omega_1^{\Delta_m} &:= \max_{j \in \Delta_m} \omega_1^j, \quad \omega_2^{\Delta_m} := \max_{j,k \in \Delta_m} \omega_2^{j,k}, \mbox{~and~} \omega_3^{\Delta_m} := \max_{j,k,\ell \in \Delta_m} \omega_3^{j,k,\ell}.
\end{align*}

We further introduce Lipschitz constants of $\partial_{\theta_j} g_{\Delta_m}$, $\partial_{\theta_j} \partial_{\theta_{k}} g_{\Delta_m}$ and $\partial_{\theta_j} \partial_{\theta_{k}} \partial_{\theta_{l}} g_{\Delta_m}$.
Rigorously, let $\zeta_{1}^{j}$, $\zeta_{2}^{j,k}$ and $\zeta_{3}^{j,k,\ell}$ be the Lipschitz constants that satisfy
\begin{align*}
    &\zeta_1^{j} \geq \frac{\norm{\partial_{\theta_j} g_{\Delta_m(\theta')} - \partial_{\theta_j} g_{\Delta_m(\theta'')}}_{L^\infty} }{\| \theta'_{\Delta_m} - \theta''_{\Delta_m} \|}, \quad 
    \zeta_2^{j,k}  \geq \frac{\norm{\partial_{\theta_j} \partial_{\theta_{k}} g_{\Delta_m(\theta')} - \partial_{\theta_j} \partial_{\theta_{k}} g_{\Delta_m(\theta'')}}_{L^\infty} }{\| \theta'_{\Delta_m} - \theta''_{\Delta_m} \|} ,
\end{align*}
and
\begin{align*}
    &\zeta_3^{j,k,\ell}  \geq \frac{\norm{\partial_{\theta_j} \partial_{\theta_{k}} \partial_{\theta_{\ell}} g_{\Delta_m(\theta')} - \partial_{\theta_j} \partial_{\theta_{k}} \partial_{\theta_{\ell}} g_{\Delta_m(\theta'')}}_{L^\infty} }{\| \theta'_{\Delta_m} - \theta''_{\Delta_m} \|},
\end{align*}
for any $\theta',\theta''$.
We also define 
\begin{align*}
    \zeta_1^{\Delta_m} := \max_{j \in \Delta_m} \zeta_1^j, \quad \zeta_2^{\Delta_m} := \max_{j,k \in \Delta_m} \zeta_2^{j,k}, \mbox{~and~} \zeta_3^{\Delta_m} := \max_{j,k,\ell \in \Delta_m} \zeta_3^{j,k,\ell}.
\end{align*}
By the setting of this problem, $\omega_1^{\Delta_m},\omega_2^{\Delta_m},\omega_3^{\Delta_m},\zeta_1^{\Delta_m},\zeta_2^{\Delta_m}$ and $\zeta_3^{\Delta_m}$ are bounded for any $m=1,...,M$.

\begin{proof}[Proof of Proposition \ref{prop:additive}]
This proof mainly consists of three steps.
    The first and second steps are for bounding the derivatives and Lipschitz constants in Assumption \ref{asmp:fisher_residual}, and the third step is for Assumption \ref{asmp:taylor_residual}.
    
    \textbf{Step (i). Bound for $\sum_j (\beta_p^j)^2$ on derivatives}:
    As preparation, we derive an explicit form of the higher-order derivative $U^j_p(\theta,z)$.
    Fix $j \in [p]$ and let $j \in \Delta_{m'}$ for some $m' \in [M]$.
    By the form of the regression model, $U^j_p(\theta,z)$ has the following form:
    \begin{align}
        U^j_p(\theta,z)
       & = \frac{1}{\sigma^2} (y - g(x)) \partial_{\theta} \partial_{\theta}^\top \partial_{\theta_j} g(x)
        -\frac{1}{\sigma^2} (\partial_{\theta_j} g(x))\partial_{\theta} \partial_{\theta}^\top g(x) \notag \\
        &\quad -\frac{1}{\sigma^2} \partial_{\theta} \partial_{\theta_j} g(x) \partial_{\theta}^\top g(x)
        -\frac{1}{\sigma^2} \partial_{\theta} g(x) \partial_{\theta}^\top \partial_{\theta_j} g(x) \notag \\
       & = \frac{1}{M\sigma^2} (y - g_{\Delta_{m'}}(x)) \partial_{\theta} \partial_{\theta}^\top \partial_{\theta_j} g_{\Delta_{m'}(\theta)}(x)
        -\frac{1}{M\sigma^2} (\partial_{\theta_j} g_{\Delta_{m'}(\theta)}(x)) \partial_{\theta} \partial_{\theta}^\top g(x) \notag \\
        &\quad -\frac{1}{M\sigma^2} \partial_{\theta} \partial_{\theta_j} g_{\Delta_{m'}(\theta)}(x) \partial_{\theta}^\top g(x)
        -\frac{1}{M\sigma^2} \partial_{\theta} g(x) \partial_{\theta}^\top \partial_{\theta_j} g_{\Delta_{m'}(\theta)}(x). \label{eq:U_p^j}
    \end{align}
    We will develop upper bounds for operator norms of the derivatives that appeared in the form of $U_p^j(\theta,z)$, then bound an operator norm of $U_p^j(\theta,z)$.

    We derive an upper bound of an operator norm of the Hesse matrix $\partial_\theta \partial_\theta^\top g(x)$.
    Let $G_2^{\Delta_m}(x)$ be a $\card(\Delta_m) \times \card(\Delta_m)$ matrix which corresponds to
    $\qty(\partial_\theta \partial_\theta^\top g(x))_{j\in\Delta_m, k\in\Delta_m}$.
    Since $\partial_\theta \partial_\theta^\top g(x)$ is a block diagonal matrix by the additive structure of $g_\theta$, we can rewrite its operator norm as
    \begin{align*}
        \norm{\partial_\theta \partial_\theta^\top g(x)}_{op} = \max\qty{\norm{G_2^{\Delta_1}(x)}_{op}, \dots, \norm{G_2^{\Delta_{M}}(x)}_{op}}.
    \end{align*}
    We can bound each $\|{G_2^{\Delta_m}(x)}\|_{op}, m=1,...,M$ as
    \begin{align*}
        \norm{G_2^{\Delta_m}(x)}_{op} &\leq \frac{1}{M} \sqrt{\sum_{\ell \in\Delta_m} \sum_{k \in \Delta_m} \qty(\omega_2^{\ell,k})^2} \leq \frac{\card(\Delta_m)}{M} \omega_2^{\Delta_m}.
    \end{align*}
    Combining the results, we obtain the following upper bound for the operator norm
    \begin{align}
        \norm{\partial_\theta \partial_\theta^\top g(x)}_{op} &\leq \frac{p + M}{M^2} \max_{m''\in[M]} \omega_2^{\Delta_{m''}}. \label{ineq:bound_deriv1}
    \end{align}
    
    Based on the inequality \eqref{ineq:bound_deriv1}, we can develop an upper bound for the other derivatives.
    Using a relation $\norm{A}_{op} \leq \norm{A}_{F}$ for any matrix $A$, we obtain the followings:
    \begin{align}
        &\norm{\partial_{\theta} g(x) \partial_{\theta}^\top g(x) \partial_{\theta_j} g_{\Delta_{m}(\theta)}(x)}_{op} \leq \omega_1^j \sqrt{\sum_{l\in\Delta_{m}} \sum_{k\in\Delta_{m}} \qty(\omega_1^{l} \omega_1^{k})^2}\leq \frac{p + M}{M} \qty(\omega_1^{\Delta_{m}})^3 \label{ineq:bound_deriv2}\\
        &\norm{\partial_{\theta} \partial_{\theta}^\top \partial_{\theta_j} g_{\Delta_{m}(\theta)}(x)}_{op} \leq \sqrt{\sum_{\ell\in\Delta_{m}} \sum_{k\in\Delta_{m}} \qty(\omega_3^{\ell, k, j})^2} \leq \frac{p + M}{M} \omega_3^{\Delta_{m}}, \label{ineq:bound_deriv3}
    \end{align}
    and
    \begin{align}
        \norm{\partial_{\theta} \partial_{\theta_j} g_{\Delta_{m}(\theta)}(x) \partial_{\theta}^\top g(x)}_{op}\leq \frac{1}{M}\sqrt{\sum_{\ell\in\Delta_{m}} \sum_{k\in[p]} \qty(\omega_2^{\ell, j} \omega_1^{k})^2} \leq \frac{p+M}{M^{3/2}} \omega_2^{\Delta_m} \max_{m''\in[M]} \omega_1^{\Delta_{m''}}. \label{ineq:bound_deriv4}
    \end{align}
    
    We put  \eqref{ineq:bound_deriv2}, \eqref{ineq:bound_deriv3} and \eqref{ineq:bound_deriv4} into an operator norm of $U^j_p(\theta,z)$ associated with the form \eqref{eq:U_p^j}.
    By the setting $M / p^{3/4} \to \infty$, we can assume $M / p^{3/4} \geq 1$ without loss of generality.
    Then, we obtain
    \begin{align*}
        &\sup_{\theta} \sup_{z} \norm{U^j_p(\theta,z)}_{op}\\
        &\leq \frac{p+M}{M^2} \biggl(b (\omega_1^{\Delta_{m'}})^3 + b \omega_3^{\Delta_{m'}} + \frac{1}{M} \omega_1^{\Delta_{m'}} \max_{m''\in[M]} \omega_2^{\Delta_{m''}}+ 2\frac{1}{M^{1/2}} \omega_2^{\Delta_m} \max_{m''\in[M]} \omega_1^{\Delta_{m''}}\biggr)\\
        &= O \left( \frac{p+M}{M^2}\right),
    \end{align*}
    by the boundedness of $\omega_1^{\Delta_m},\omega_2^{\Delta_m} $ and $\omega_3^{\Delta_m}$.
    Since we have 
    \begin{align*}
        \frac{p + M}{M^2} \leq \frac{1}{p^{1/2}}\qty(\qty(\frac{p^{3/4}}{M})^2 + \frac{1}{p^{1/4}}) &= o\qty(\frac{1}{p^{1/2}}),
    \end{align*}    
    we obtain  $\sum_{j=1}^p (\beta_p^j)^2 = o(1)$.


    \textbf{Step (ii). Bound for $\sum_j (\alpha_p^j)^2$ on Lipschitz constants}:
    We consider a ${\card(\Delta_m) \times \card(\Delta_m)}$ matrix $H_2^{\Delta_m}(x) := (\partial_\theta \partial_\theta^\top (g_{\underline{\theta}(x)} - g_{\underline{\theta}'}(x)))_{j\in\Delta_m, k\in\Delta_m}$ for $\underline{\theta},\underline{\theta}'$, then develop an upper bound for an operator norm of $H_2^{\Delta_m}(x)$.
    Since $\partial_\theta \partial_\theta^\top (g_{\underline{\theta}(x)} - g_{\underline{\theta}'}(x)))$ is a block diagonal matrix, we obtain
    \begin{align}
        \norm{\partial_\theta \partial_\theta^\top \qty(g_{\underline{\theta}}(x) - g_{\underline{\theta}'}(x))}_{op} &= \max\qty{\norm{H_2^{\Delta_1}(x)}_{op}, \dots, \norm{H_2^{\Delta_{M}}(x)}_{op}} \notag \\
        &\leq \max_{m \in [M]} \frac{1}{M} \sqrt{\sum_{l\in\Delta_m} \sum_{k \in \Delta_m} \qty(\zeta_2^{l,k} \norm{\underline{\theta} - \underline{\theta}'})^2} \notag \\
        &\leq \frac{p + M}{M^2} \max_{m''\in[M]} \zeta_2^{\Delta_{m''}} \norm{\underline{\theta} - \underline{\theta}'}. \label{ineq:bound_lip1}
    \end{align}
    We apply \eqref{ineq:bound_lip1} and the bounds for derivative \eqref{ineq:bound_deriv1}, we bound the following difference of the third derivatives as
    \begin{align}
        \norm{\partial_{\theta} \partial_{\theta}^\top \partial_{\theta_j} g_{\Delta_{m'}(\underline{\theta})}(x) - \partial_{\theta} \partial_{\theta}^\top \partial_{\theta_j} g_{\Delta_{m'}(\underline{\theta}')}'(x)}_{op}& \leq \sqrt{\sum_{\ell \in\Delta_{m'}} \sum_{k\in\Delta_{m'}} \qty(\zeta_3^{\ell, k, j} \norm{\underline{\theta} - \underline{\theta}'})^2} \notag \\
        & \leq \frac{p + M}{M} \zeta_3^{\Delta_{m'}} \norm{\underline{\theta} - \underline{\theta}'}, \label{ineq:bound_lip2}
    \end{align}
    and also the difference of the following product of first and second derivatives as
    \begin{align}
        &\norm{(\partial_{\theta_j} g_{\Delta_{m'}(\underline{\theta})}(x)) \partial_{\theta} \partial_{\theta}^\top g_{\underline{\theta}}(x) - (\partial_{\theta_j} g_{\Delta_{m'}(\underline{\theta}')}'(x))\partial_{\theta} \partial_{\theta}^\top g_{\underline{\theta}'}(x)}_{op} \notag \\
        & \leq \abs{\partial_{\theta_j} g_{\Delta_{m'}(\underline{\theta}')}(x)} \norm{\partial_{\theta} \partial_{\theta}^\top g_{\underline{\theta}}(x) - \partial_{\theta} \partial_{\theta}^\top g_{\underline{\theta}'}(x)}_{op} \notag \\
        &  \quad + \norm{\partial_{\theta} \partial_{\theta}^\top g_{\underline{\theta}}(x)}_{op} \abs{\partial_{\theta_j} g_{\Delta_{m'}(\underline{\theta})}(x) - \partial_{\theta_j} g_{\Delta_{m'}(\underline{\theta}')}(x)} \notag \\
        & \leq \frac{p + M}{M^2} \qty(\omega_1^{\Delta_{m'}} \max_{m''\in[M]} \zeta_2^{\Delta_{m''}} + \zeta_1^{\Delta_{m'}} \max_{m''\in[M]} \omega_2^{\Delta_{m''}}) \norm{\underline{\theta} - \underline{\theta}'}, \label{ineq:bound_lip3}\\
        & \mbox{and } \notag \\
        &\norm{\partial_{\theta} \partial_{\theta_j} g_{\Delta_{m'}(\underline{\theta})}(x) \partial_{\theta}^\top g_{\underline{\theta}}(x) - \partial_{\theta} \partial_{\theta_j} g_{\Delta_{m'}(\underline{\theta})}(x) \partial_{\theta}^\top g_{\underline{\theta}'}(x)}_{op} \notag \\
        & \leq \norm{\partial_{\theta} \partial_{\theta_j} g_{\Delta_{m'}(\underline{\theta})}(x) \partial_{\theta}^\top g_{\underline{\theta}}(x) - \partial_{\theta} \partial_{\theta_j} g_{\Delta_{m'}(\underline{\theta})}(x) \partial_{\theta}^\top g_{\underline{\theta}'}(x)}_{op} \notag \\
        & \quad + \norm{\partial_{\theta} \partial_{\theta_j} g_{\Delta_{m'}(\underline{\theta})}(x) \partial_{\theta}^\top g_{\underline{\theta}'}(x) - \partial_{\theta} \partial_{\theta_j} g_{\Delta_{m'}(\underline{\theta}')}(x) \partial_{\theta}^\top g_{\underline{\theta}'}(x)}_{op} \notag \\
        & \leq \sqrt{\sum_{\ell\in\Delta_{m'}} \sum_{k\in[p]} \qty(\omega_2^{\ell,j} \zeta_1^{k} \norm{\underline{\theta} - \underline{\theta}'})^2} + \sqrt{\sum_{\ell\in\Delta_{m'}} \sum_{k\in[p]} \qty(\zeta_2^{\ell,j} \omega_1^{k} \norm{\underline{\theta} - \underline{\theta}'} )^2} \notag \\
        & \leq \frac{p+M}{M^{3/2}} \qty(\omega_2^{\Delta_{m'}} \max_{m''\in[M]} \zeta_1^{\Delta_{m''}} + \zeta_2^{\Delta_{m'}} \max_{m''\in[M]} \omega_1^{\Delta_{m''}}) \norm{\underline{\theta} - \underline{\theta}'}. \label{ineq:bound_lip4}
    \end{align}
    
    We combine the form \eqref{eq:U_p^j} and the derivative bounds \eqref{ineq:bound_lip1}, \eqref{ineq:bound_lip2}, \eqref{ineq:bound_lip3}, \eqref{ineq:bound_lip4}, then obtain
    \begin{align*}
        &\sup_{\underline{\theta}, \underline{\theta}'} \sup_{z} \norm{U^j_p(\underline{\theta},z) - U^j_p(\underline{\theta}',z)}_{op}\\
        & \leq \frac{p + M}{M^2 \sigma^2} \norm{\underline{\theta} - \underline{\theta}'} \Biggl(b \zeta_3^{\Delta_{m'}} + \frac{1}{M} \qty(\omega_1^{\Delta_{m'}} \max_{m''\in[M]} \zeta_2^{\Delta_{m''}} + \zeta_1^{\Delta_{m'}} \max_{m''\in[M]} \omega_2^{\Delta_{m''}})\\
        &\quad \quad + \frac{2}{M^{1/2}} \qty(\omega_2^{\Delta_{m'}} \max_{m''\in[M]} \zeta_1^{\Delta_{m''}} + \zeta_2^{\Delta_{m'}} \max_{m''\in[M]} \omega_1^{\Delta_{m''}})\Biggr) \\
        &= O\left( \frac{p + M}{M^2}\right),
    \end{align*}
    by the boundedness of the coefficients $\omega_1^{\Delta_{m'}}, \omega_2^{\Delta_{m'}} \zeta_1^{\Delta_m},\zeta_2^{\Delta_m}, \zeta_3^{\Delta_m}$, noise variance $\sigma^2$, and the parameter space.
    Since $(p + M)/M^2 = o(1/\sqrt{p})$ holds, we obtain  $\sum_j (\alpha_p^j)^2 = o(1)$.
    
    \textbf{Step (iii). Bound for Assumption \ref{asmp:fisher_residual}}:
    We derive an explicit form of $w(z_i)$ in this setting:
    \begin{align*}
         w(z_i) &= (1/f(z_i))\partial_\theta \partial_\theta^\top f(z_i)\\
         &= \frac{1}{\sigma^4} \qty((y_i - g_{\theta}(x_i))^2 - \sigma^2) \partial_\theta g_{\theta}(x_i) \partial_\theta^\top g(x_i) + \frac{1}{\sigma^2}(y_i - g_{\theta}(x_i)) \partial_\theta \partial_\theta^\top g_{\theta}(x_i).
    \end{align*}
    Since the Hesse matrix is written as
    \begin{align*}
        \partial_\theta g_{\theta}(x_i) \partial_\theta^\top g_{\theta}(x_i) &= \frac{1}{M^2} \sum_{m'' \in [M]} \partial_\theta g_{\Delta_{m''}(\theta)}(x_i) \partial_\theta^\top g_{\Delta_{m''}(\theta)}(x_i),
    \end{align*}
    we can obtain the following inequality:
    \begin{align*}
        \norm{\partial_\theta g_\theta(x_i) \partial_\theta^\top g_\theta(x_i)}_{op} \leq \frac{1}{M^2} \frac{p + M}{M} \max_{m''\in[M]} \qty(\omega_1^{\Delta_{m''}})^2.
    \end{align*}
    Then, we obtain the following bound by using the boundedness of the parameters and noise variance and $p/M = o(1)$:
    \begin{align*}
        \norm{w(z_i)}_{op} \leq \frac{p + M}{M^2\sigma^4} \qty( ( b^2 + \sigma^2 ) \frac{1}{M} \max_{m''\in[M]} \qty(\omega_1^{\Delta_{m''}})^2 + b \sigma^2 \max_{m''\in[M]} \omega_2^{\Delta_{m''}}) = o(1).
    \end{align*}
    Also, we obtain $\norm{\mathbb{E}[w(z_i) w(z_i)^\top]}_{op} \leq \mathbb{E}[\norm{w(z_i)}_{op}^2] = o(1)$.
\end{proof}

\bibliographystyle{plain}
\bibliography{main}

\end{document}